\documentclass[lettersize,journal]{IEEEtran}
\usepackage{amsmath,amsfonts}
\usepackage{array}
\usepackage[caption=false,font=normalsize,labelfont=sf,textfont=sf]{subfig}
\usepackage{textcomp}
\usepackage{stfloats}
\usepackage{url}
\usepackage{verbatim}
\usepackage{graphicx}
\def\BibTeX{{\rm B\kern-.05em{\sc i\kern-.025em b}\kern-.08em
    T\kern-.1667em\lower.7ex\hbox{E}\kern-.125emX}}
\usepackage{balance}

\usepackage{amsmath}

\usepackage{easyReview}
\usepackage{balance}

\usepackage{adjustbox}
\usepackage{threeparttable}
\usepackage{booktabs}
\usepackage{amssymb}
\usepackage{multirow}
\usepackage{booktabs}

\usepackage{subcaption}
\usepackage{flushend}
\usepackage{ragged2e}
\usepackage{url}
\usepackage{threeparttable}
\usepackage{xspace}
\usepackage{amsmath}
\usepackage{algpseudocode}
\usepackage{tablefootnote}
\usepackage{footnote}
\usepackage{amsthm}
\usepackage{enumitem}
\usepackage{dsfont}
\theoremstyle{plain}
\newtheorem{theorem}{Theorem}[section] 

\makeatletter
\newif\if@restonecol
\makeatother

\usepackage[linesnumbered,ruled,vlined]{algorithm2e}
\newtheorem{corollary}[theorem]{Corollary}

\DeclareMathOperator*{\argmax}{arg\,max}

\newtheorem{definition}{Definition}

\usepackage{color}

\begin{document}
\title{{CombAlign}: Enhancing Model Expressiveness in Unsupervised Graph Alignment}
\author{Songyang~Chen,
Yu~Liu, 
Lei~Zou,
Zexuan~Wang,
and~Youfang~Lin

\IEEEcompsocitemizethanks{\IEEEcompsocthanksitem Songyang Chen, Yu Liu, Zexuan Wang, and Youfang Lin are with the School of Computer Science and Technology, Beijing Jiaotong University, Beijing 100044, China, and the Beijing Key Laboratory of Traffic Data Mining and Embodied Intelligence, Beijing 100044, China. 
Lei Zou is with Peking University, Beijing 100871, China. 
Yu Liu is the corresponding author. 
\protect\\
E-mail: songyangchen@bjtu.edu.cn; yul@bjtu.edu.cn;  zoulei@pku.edu.cn; zexuanwang@bjtu.edu.cn; yflin@bjtu.edu.cn.
}
\thanks{Manuscript received xx xx, xxxx; revised xx xx, xxxx.}
}

\markboth{Journal of \LaTeX\ Class Files,~Vol.~18, No.~9, September~2020}%
{How to Use the IEEEtran \LaTeX \ Templates}

\maketitle

\begin{abstract}
Unsupervised graph alignment finds the node correspondence between a pair of attributed graphs by only exploiting graph structure and node features. 
One category of recent studies first computes the node representation 
and then matches nodes with the largest embedding-based similarity, while the other category reduces the problem to optimal transport (OT) via Gromov-Wasserstein learning. 
However, it remains largely unexplored in the model expressiveness, as well as how theoretical expressivity impacts prediction accuracy. 
We investigate the model expressiveness from two aspects. 
First, we characterize the model’s \emph{discriminative power} in distinguishing matched and unmatched node pairs across two graphs. 
Second, we study the model’s capability of guaranteeing \emph{node matching properties} such as {one-to-one} matching and mutual alignment. 
Motivated by our theoretical analysis, we put forward a hybrid approach named \texttt{CombAlign} with stronger expressive power. 
Specifically, we enable cross-dimensional feature interaction for OT-based learning and propose an embedding-based method inspired by the Weisfeiler-Lehman test. 
We also apply non-uniform marginals obtained from the embedding-based modules to OT as priors for more expressiveness. 
Based on that, we propose a traditional algorithm-based refinement, which combines our OT and embedding-based predictions using the ensemble learning strategy and reduces the problem to maximum weight matching. 
With carefully designed edge weights, we ensure these matching properties and further enhance prediction accuracy.
By extensive experiments, we demonstrate a significant improvement of 14.5\% in alignment accuracy compared to state-of-the-art approaches and confirm the soundness of our theoretical analysis.
\end{abstract}


\section{Introduction}\label{sec:introduction}
\IEEEPARstart{T}{he} unsupervised graph alignment problem predicts the node correspondence between two attributed graphs without any pre-aligned node pairs (i.e., anchors). 
It has a wide range of applications, including linking the same identity across different social networks~\cite{zhang2015multiple,li2019partially}, matching scholar accounts between multiple academic platforms~\cite{tang2008arnetminer,zhang2021balancing}, matching the protein networks from different species~\cite{gwl,singh2008global}, computing the graph edit distance between graphs~\cite{piao2023computing}, and various computer vision tasks~\cite{bernard2015solution,haller2022comparative}. 
For instance, cross-species analysis of biomolecular networks finds the node matching between gene co-expression networks, protein-protein interaction networks, or metabolic networks in the unsupervised setting~\cite{zhou2022metabolite}. 
In real-world scenarios, given the substantial structural disparities among the nodes to be aligned, as well as the difficulty in obtaining pre-aligned node pairs, unsupervised graph alignment is both important and challenging.

\begin{figure*}[t]
 \centering
    \begin{tabular}{ccc}
        \hspace{0mm} \includegraphics[height=27mm]{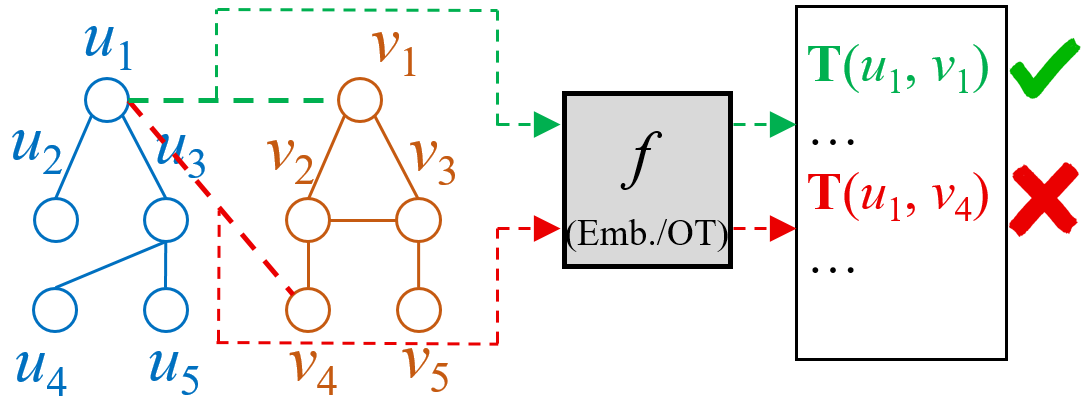} &
        \hspace{-2mm} \includegraphics[height=27mm]{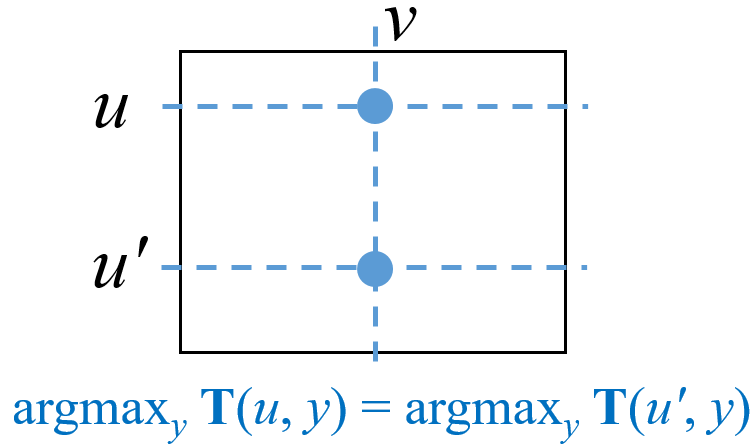} &
        \hspace{-4mm} \includegraphics[height=30mm]{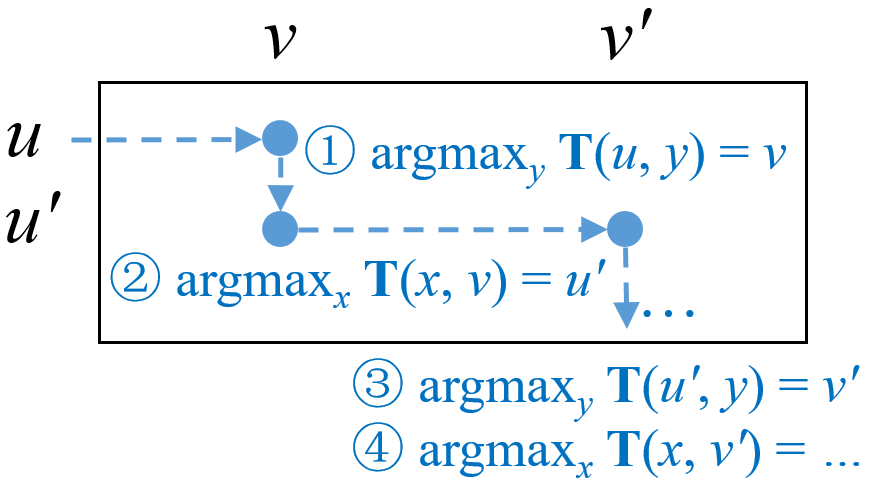} \\

       \hspace{0mm} (a) {Ranking matched and unmatched node pairs} &
       \hspace{0mm} (b) {One-to-many prediction} &
       \hspace{0mm} (c) {Mutual alignment} \\
 \end{tabular}
\vspace{-2mm}
\caption{Illustration of model expressiveness for graph alignment, where (a) denotes the model's capability of separating matched and unmatched node pairs, while (b)~\&~(c) showcase two disadvantages of pure learning-based approaches, i.e., one-to-many prediction regardless of the one-to-one matching constraint, and the inconsistency between row and column-wise alignment.} \label{fig:intro}
\end{figure*}

Significant research efforts have been dedicated to the unsupervised graph alignment problem. 
Early work~\cite{feizi2019spectral,karakasis2021joint,peng2010new} often formulates the task as the maximum common subgraph isomorphism or the quadratic assignment problem, of which the complexities are NP-hard. In recent years, learning-based methods have garnered increasing attention, surpassing traditional solutions in terms of alignment accuracy. 
One prominent category of them~\cite{galign,walign,gtcalign} relies on the \emph{``embed-then-cross-compare''} paradigm. 
These methods first learn the embedding of each node in both graphs by exploiting the graph structure and node feature information, for example, with graph neural networks (GNNs). 
Then, two nodes are matched if their embeddings are close according to some specific metric such as cosine similarity~\cite{gtcalign}. 
However, for the unsupervised scenario, it is non-trivial to design an optimization objective that is fully suitable for the graph alignment task without any known node correspondence, which poses challenges for learning high-quality embeddings.
The most recent work~\cite{gtcalign} employs a parameter-free approach without an explicit objective to achieve state-of-the-art performance. 
Another recent line of research~\cite{gwl,slotalign,fusegwd,dhot, HLOT} generally models the graph alignment problem via \emph{optimal transport (OT)} and achieves promising accuracy due to its well-defined objective. 
Given the marginal distributions on two finite sets and a predefined transformation cost between them, the OT problem finds the joint distribution to minimize the total transport cost~\cite{WD}. 
It has been shown that unsupervised graph alignment can be reduced to OT with the help of Gromov-Wasserstein (GW) discrepancy~\cite{gwl,slotalign,peyre2016gromov}, which shares similar ideas with node alignment~\cite{slotalign,gwl}. 
For approaches of this paradigm, the crux lies in
the definition of pairwise cost within each graph (i.e., \emph{intra-graph} cost), then the transport cost (i.e., \emph{inter-graph} cost) is computed accordingly. 
We have witnessed an increasingly sophisticated design of intra-graph cost~\cite{gwl, slotalign, dhot} for better matching accuracy. 
Unfortunately, it remains largely unexplored for the \emph{theoretical capability} of deep learning-based solutions. 

\noindent \textbf{Discriminative Power}. 
We first investigate the model expressiveness for unsupervised graph alignment in the ability to distinguish matched and unmatched node pairs. 
As shown in Figure~\ref{fig:intro}(a) where $u_i$ is assumed to align to $v_i$ for $i \in \{1, \ldots, 5\}$, we see both embedding and optimal transport-based methods as learning a (parameterized) function $f$ that maps two nodes in different graphs to their matching probability. We generalize existing approaches~\cite{slotalign,gtcalign,galign,gwl} and focus on \emph{ranking matched node pairs above the unmatched ones by the matching probabilities}. 
Please note that the model's accuracy is closely related to the discriminative power of $f$.

\noindent \textbf{Node Matching Properties}.
We regret to find that two major disadvantages exist with pure learning-based approaches, which violate the property of graph alignment and deteriorate the accuracy. 
On the one hand, we note that the \emph{one-to-one} matching constraint is forced in scenarios such as the alignment of the protein-protein interaction networks~\cite{kolavr2012graphalignment} because the conserved functional modules must be uniquely mapped. 
Although being the \emph{de facto} setting for most studies~\cite{walign,dhot,slotalign,gtcalign,gwl,zeng2023parrot,zhang2016final,yan2021bright,liu2023wl,ning2023graph,huynh2021network}, these methods typically produce one-to-many predictions (Figure~\ref{fig:intro}(b)) since they compute a matrix of matching probability for every pair of nodes and then take the row/column-wise maximum as the prediction.  
On the other hand, as aligned nodes are paired up in nature, the alignment is \emph{mutual} (i.e., bidirectional). 
However, predicting the alignment directly via a probability-based approach might lead to inconsistency. 
As shown in Figure~\ref{fig:intro}(c), for node $u$, its matched node (i.e., $v$) should be aligned to another node $u'$ according to the alignment probabilities. 
We tackle these disadvantages of pure learning-based approaches to enhance model's theoretical soundness and prediction accuracy.

Motivated by our viewpoint of model expressiveness, we improve the graph alignment process in the following aspects, and present a model framework named \texttt{CombAlign} to unify them in a principled way.  

\noindent \emph{Enhancing the discriminative power of OT-based learning}. 
First, we enhance the optimal transport-based paradigm~\cite{slotalign,dhot,gwl}, which has a well-defined objective, by incorporating \emph{feature transformation} along with the well-adopted feature propagation~\cite{slotalign, dhot} in the transport cost design of GW learning. 
We prove that more discriminative power is obtained by enabling the cross-dimension mixture of features, which helps to distinguish matched and unmatched node pairs. 
As shown in Section~\ref{sec:experiments}, with this single optimization, our model outperforms existing OT and embedding-based state-of-the-art. 

\noindent \emph{Enhancing the discriminative power with WL}. 
Second, instead of proposing complicated heuristics (e.g.,~\cite{gtcalign}) under the embedding-based framework, we opt for a simple yet effective approach by using parameter-free GNNs to simulate the Weisfeiler-Lehman (WL) algorithm~\cite{wl-test}. 
We also exploit {embedding-based prior knowledge} to improve the well-formed OT objective, more specifically, we compute \emph{non-uniform marginals} for OT, a vital input to the learning process.  
The learning algorithm is proved to be more expressive by optimizing an orthogonal aspect to feature transformation. 

\noindent \emph{Guaranteeing node matching properties with improved accuracy}.
Third, to solve the disadvantages in Figure~\ref{fig:intro}(b)~\&~(c), we reduce graph alignment to the \emph{maximum weight matching} problem~\cite{km,km2} by properly establishing a set of weighted edges between two node sets. 
To define the edge weights, which determine the matching accuracy, we combine the predictions of our OT and embedding-based procedures via an {ensemble learning} strategy~\cite{zhou2012ensemble,dietterich2000ensemble} named \emph{stacking}. 
Our model not only guarantees the matching properties but also improves prediction accuracy by combining the best of both worlds. 

Our contributions for the unsupervised graph alignment problem are summarized as follows: 
\begin{itemize} 
\item We combine optimal transport and embedding-based approaches to improve the model expressiveness in \emph{distinguishing matched and unmatched node pairs}. In particular, we demonstrate the necessity of feature transformation and non-uniform marginals in the learning process. 

\item We transform unsupervised graph alignment into a maximum weight matching problem with an ensemble learning strategy to integrate embedding and OT-based predictions. This not only ensures important matching properties but also enhances prediction accuracy.

\item We propose \texttt{CombAlign}, a unified framework for unsupervised graph alignment to fully integrate the advantages of embedding-based, OT-based, and traditional algorithm-based solutions.

\item Extensive experiments are conducted to evaluate our algorithm against both embedding and OT-based state-of-the-art. Our method improves alignment accuracy by a significant margin, while a detailed evaluation shows the effectiveness of each proposed module.
\end{itemize}

\section{Related Work}\label{related}
\noindent \textbf{Unsupervised Graph Alignment}. 
Classic graph alignment methods, such as \texttt{EigenAlign}~\cite{feizi2019spectral}, formulates the problem as a Quadratic Assignment Problem~\cite{karakasis2021joint}, which considers both matches and mismatches and solves it by spectral decomposition of matrices. 
Meanwhile, \texttt{LREA}~\cite{nassar2018low} addresses the problem by solving a maximum weight bipartite matching problem on a low-rank approximation of a node similarity matrix, and \texttt{GRASP}~\cite{hermanns2021grasp} treats the problem as a special case of finding a mapping between functions on graphs with the linear assignment algorithm. Recently, the research focus has shifted towards learning-based methods, which predict an alignment matrix indicating the matching probability for every node pair. 
Existing unsupervised solutions can be roughly divided into two categories. 
The first category, referred to as \emph{``embed-then-cross-compare”}~\cite{galign,walign,gtcalign,liu2016aligning,sun2023towards,sana}, tackles the problem by first generating node representation for both graphs, e.g., using GNNs. 
Then, the alignment probability is computed by specific similarity measures based on node embeddings, e.g., the cosine similarity~\cite{gtcalign,galign}. 
In particular, \texttt{GAlign}~\cite{galign} incorporates the idea of data augmentation into the learning objective to obtain high-quality node embeddings. 
\texttt{GTCAlign}~\cite{gtcalign} further simplifies \texttt{GAlign} by using a GNN with {randomized parameters} for node embedding computation.  
Another recent work~\cite{walign} conducts unsupervised graph alignment via the idea of generative adversarial networks (GANs).

Instead of designing sophisticated alignment rules based on node embeddings, the second category is known as \emph{optimal transport (OT)}-based approaches. 
They share the well-defined objective of minimizing Gromov-Wasserstein discrepancy, which predicts the alignment probabilities given the transport cost between node pairs. 
These methods mainly differ in the specific form of the \emph{learnable} transport cost. 
\texttt{GWL}~\cite{gwl} is the first work following this paradigm with predefined cost and incorporates Wasserstein discrepancy and reconstruction loss for regularization. 
As the state-of-the-art approach, \texttt{SLOTAlign}~\cite{slotalign} extends \texttt{GWL} by multi-view structural modeling, which includes the linear combination of adjacency information, node features, and {feature propagation}. 
\texttt{UHOT-GM}~\cite{dhot} generalizes this idea with the notion of multiple and cross-modal alignment. The most recent work, \texttt{HLOT}~\cite{HLOT}, introduces hypergraph formulation to enrich inter-node relationships by modeling higher-order dependencies.

\noindent \textbf{Semi-Supervised Graph Alignment}. Unlike unsupervised graph alignment, with a set of anchor node pairs as input, the main idea of semi-supervised methods is to fully utilize the anchor information. 
A line of consistency-based algorithms~\cite{zhang2016final,zhang2021balancing,singh2008global,koutra2013big} are proposed to first define the neighborhood topology and attribute consistency followed by devising a procedure to propagate the anchor information across graphs, e.g., via Random Walk with Restart (RWR). 
Follow-up works resort to embedding-based approaches~\cite{zhang2020nettrans,yan2021bright,heimann2018regal,zhou2021attent,hong2020domain,liu2023wl,ning2023graph,huynh2021network}, i.e., by learning node embeddings to make the anchor pairs close while preserving the graph structure in the latent space. 
Recently, several OT-based methods have attracted considerable attention. \cite{zeng2023parrot} leverages optimal transport to model the problem through alignment consistency, while \cite{chen2025enhancing} employs adaptive optimal transport by designing a learnable cost matrix, thereby achieving improved accuracy.

\noindent \textbf{Knowledge Graph Entity Alignment}. A similar problem has been extensively studied~\cite{wang2018cross,  sun2018bootstrapping,  zeng2021comprehensive}, which aligns the entities across knowledge graphs~\cite{zhang2022benchmark}. 
For this problem setting, both topological structures and feature semantics are important. Building upon optimal transport,  \texttt{FGWEA}~\cite{tang2023fused} proposes a fused Gromov-Wasserstein framework for unsupervised knowledge graph entity alignment, which effectively integrates structural and semantic information to achieve robust performance.
Recent advances~\cite{jiang2024unlocking,AdaCoAgentEA,LLM4EA} have explored leveraging large language models (LLMs) to enhance entity alignment with richer textual information from diverse corpora. Considering their different problem inputs, we regard recent LLM-based work as a distinct line of research. 

\section{Preliminary} \label{prelim}

\begin{table}[!t]
\centering
\caption{Table of notations.}\label{tbl:def-notation}
\small

\resizebox{\linewidth}{!}{%
\begin{tabular}{p{1.7cm} p{6.6cm}}
\toprule
\textbf{Notation} & \textbf{Description} \\
\midrule
$\mathcal{G}_s, \mathcal{G}_t$  &  The source and target graphs \\ 
$n, m$  &  The sizes of the nodes and edges \\
$\mathbf{A}_p, \mathbf{X}_p$    &  Adjacency matrix, node feature matrix ($p=s,t$) \\ 
$u_i, u_j, v_k, v_l$            &  Nodes with $u_i, u_j \in \mathcal{G}_s$ and $v_k, v_l \in \mathcal{G}_t$ \\ 
$\mathbf{Z}_p, \mathbf{H}_p, \mathbf{R}_p$   &  Embeddings computed by GNN, GNN w/o training, and feature propagation, respectively ($p=s,t$) \\ 
$\mathbf{C}_s, \mathbf{C}_t, \mathbf{C}_{\mathrm{gwd}}$    &  Intra-graph and inter-graph costs for GW learning \\ 
$\mathbf{T}_{\mathrm{WL}}, \mathbf{T}_{\mathrm{GW}}$        &   The alignment probability matrices  \\ 
$f(\cdot), g(\cdot)$    & Functions with and without learnable parameters \\ 
\bottomrule
\end{tabular}%
}
\end{table}

\noindent \textbf{Problem Statement}.
We denote an undirected and attributed graph as $\mathcal{G} = (\mathcal{V}, \mathcal{E}, \mathbf{X})$ with node set $\mathcal{V}$ of size $n$, edge set $\mathcal{E}$ represented by the adjacency matrix $\mathbf{A} \in \{0,1\}^{n \times n}$, and  
$\mathbf{X} \in \mathbb{R}^{n \times d}$ corresponding to the node feature matrix with $d$-dimensional attribute vectors.  
We list the frequently used notations in Table~\ref{tbl:def-notation}.
\begin{definition} [Unsupervised Graph Alignment]
Given source graph $\mathcal{G}_s$ and target graph $\mathcal{G}_t$, without any known anchor node pairs, unsupervised graph alignment returns a set of matched node pairs $\mathcal{M}$. 
For each node pair $(u_i, v_k) \in \mathcal{M}$, we have $u_i \in \mathcal{G}_s$ and $v_k \in \mathcal{G}_t$.

\end{definition}

We assume that $n_1 = |\mathcal{V}_s|, n_2 = |\mathcal{V}_t|$, and $n_1 \leq n_2$ w.l.o.g. 
In particular, we focus on the setting of \emph{one-to-one} node alignment, following most existing studies~\cite{yan2021bright, zeng2023parrot, walign, galign, gtcalign, gwl, slotalign, dhot}. 
That is, every node can appear at most once in $\mathcal{M}$. 
Instead of directly computing $\mathcal{M}$, existing learning-based approaches~\cite{walign, galign, gtcalign, gwl, slotalign, dhot} predict an \emph{alignment probability matrix} $\mathbf{T}$ of size $n_1 \times n_2$, where $\mathbf{T}(i, k)$ denotes the probability that $u_i \in \mathcal{G}_s$ is matched to $v_k \in \mathcal{G}_t$. 
Next, it is sufficient to set $\mathcal{M}(u_i) = \argmax_{k} \mathbf{T}(i, k)$. Similar to existing studies~\cite{slotalign,dhot,galign,walign,gtcalign}, this works primarily focuses on attributed graphs.

\noindent \textbf{Gromov-Wasserstein (GW) Learning}. The discrete form of the Gromov-Wasserstein discrepancy is defined as follows.

\begin{definition}[Gromov-Wasserstein Discrepancy (GWD)~\cite{gwl}]\label{def:GWD}
Given the distribution $\boldsymbol{\mu}$ (resp. $\boldsymbol{\nu}$) over $\mathcal{V}_s$ (resp. $\mathcal{V}_t$), the GW discrepancy between $\boldsymbol{\mu}$ and $\boldsymbol{\nu}$ is defined as
\begin{align} \label{eqn:GWD}
\min_{\mathbf{T} \in \Pi(\boldsymbol{\mu}, \boldsymbol{\nu})}& \sum_{i=1}^{n_1} \sum_{j=1}^{n_1} \sum_{k=1}^{n_2} \sum_{l=1}^{n_2} { {|\mathbf{C}_s(i, j) - \mathbf{C}_t(k, l)|}^2 \mathbf{T}(i, k) \mathbf{T}(j, l) }, \\
& \text{s.t.} \mathbf{T}\boldsymbol{1}_{n_2} = \boldsymbol{\mu}, \mathbf{T}^{\intercal} \boldsymbol{1}_{n_1} = \boldsymbol{\nu}. \nonumber
\end{align}
\end{definition}
Here, $\mathbf{C}_s \in \mathbb{R}^{n_1 \times n_1}$ and $\mathbf{C}_t \in \mathbb{R}^{n_2 \times n_2}$ are the \emph{intra-graph costs} for $\mathcal{G}_s$ and $\mathcal{G}_t$, respectively, which measure the similarity (or distance) of two nodes {within each graph}~\cite{peyre2016gromov}. 
We have $\sum_{i=1}^{n_1} \sum_{k=1}^{n_2} {\mathbf{T}(i, k)} = 1$ according to the constraints in Equation~\ref{eqn:GWD}, i.e., $\mathbf{T}$ is the joint probability distribution over two node sets, and $\boldsymbol{1}_n$ denotes the all-ones vector in $\mathbb{R}^n$. 
Let $\mathbf{C}_{gwd} \in \mathbb{R}^{n_1 \times n_2}$ be the \emph{inter-graph cost matrix}~\cite{gwl, peyre2016gromov}:
\begin{equation} \label{eqn:GWD-ik}
\mathbf{C}_{gwd}(i, k) = \sum_{j=1}^{n_1} \sum_{l=1}^{n_2} { {|\mathbf{C}_s(i, j) - \mathbf{C}_t(k, l)|}^2 \mathbf{T}(j, l) }.
\end{equation}
The above definition of $\mathbf{C}_{gwd}$ can be interpreted as follows. 
Noticing that Equation~\ref{eqn:GWD} can be reformulated as 
\begin{equation} \label{eqn:GWD-2}
\langle \mathbf{C}_{gwd}, \mathbf{T}\rangle = \sum_{i=1}^{n_1} \sum_{k=1}^{n_2} \mathbf{C}_{gwd}(i, k) \mathbf{T}(i, k).
\end{equation}
To minimize this objective with the constraints of $\mathbf{T}$, a negative correlation between the values of $\mathbf{C}_{gwd}(i, k)$ and $\mathbf{T}(i, k)$ is encouraged~\cite{WD}. More precisely, for likely matched node pairs $(u_i, v_k)$ and $(u_j, v_l)$, ${|\mathbf{C}_s(i, j) - \mathbf{C}_t(k, l)|}^2$ should be small, i.e., the values of $\mathbf{C}_s(i, j)$ and $\mathbf{C}_t(k, l)$ are close~\cite{gwl, slotalign}.

\noindent \textbf{The Optimal Transport (OT) Problem}. Actually, Equation~\ref{eqn:GWD-2} can be interpreted via optimal transport. 
Given $\mathbf{C} \in \mathbb{R}^{n_1 \times n_2}$ which represents the cost of transforming probability distribution $\boldsymbol{\mu}$ to $\boldsymbol{\nu}$, we compute the joint probability distribution $\mathbf{T} \in \mathbb{R}^{n_1 \times n_2}$ so that the total transport cost is minimized. 
It can be solved by the Sinkhorn algorithm~\cite{sinkhorn,sinkhorn-knopp} via an iterative procedure. 
With learnable cost, existing solutions~\cite{gwl, slotalign, dhot} adopt the proximal point method~\cite{ppm} to reduce GW learning to OT and to learn the alignment probability and parameters in the cost term jointly.

\section{Our Proposed Model} \label{sec:model}

\subsection{Model Overview} \label{sec:model-overview}

\begin{figure*}
    \centering
    \includegraphics[width=\linewidth]{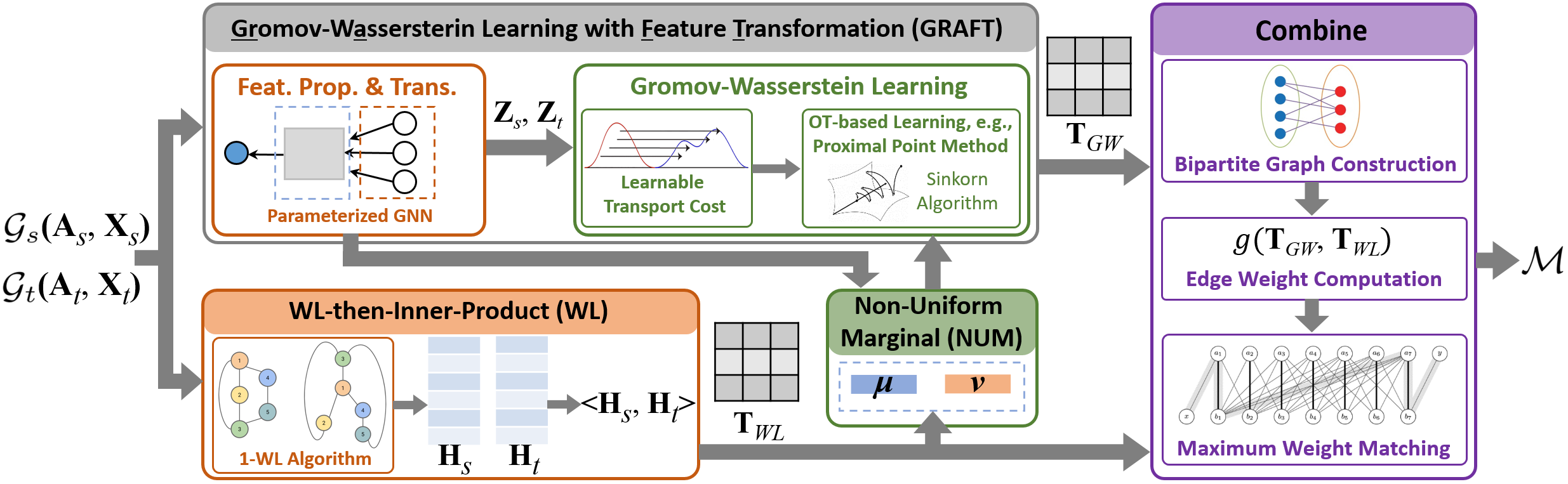}
    
    \caption{The overall framework of \texttt{CombAlign}. Modules in orange, green, and purple belong to the embedding-based, OT-based, and traditional algorithm-based approaches, respectively.} 
    \label{fig:framework}
\end{figure*}

As shown in Figure~\ref{fig:framework}, given two graphs $\mathcal{G}_s$ and $\mathcal{G}_t$, \texttt{CombAlign} first predicts the node alignment by two separate approaches. 
    The \emph{WL-then-Inner-Product} (\texttt{WL}) module computes the embedding-based alignment probability $\mathbf{T}_{WL}$ by simulating the Weisfeiler-Lehman (WL) algorithm. 
    Specifically, we employ parameter-free graph neural networks (GNNs) to derive node representations $\mathbf{H}_s$ and $\mathbf{H}_t$, and then compute their cosine similarity to obtain $\mathbf{T}_{WL}$. 
    On the other hand, the \emph{Gromov-Wasserstein Learning with Feature Transformation} (\texttt{GRAFT}) module computes an OT-based prediction $\mathbf{T}_{GW}$ with feature propagation~\&~transformation followed by GW learning. 
    It also interacts with the \emph{Non-Uniform Marginal} (\texttt{NUM}) module to obtain the marginal distributions over the two node sets, which partially relies on the \texttt{WL} module. 
    Given both predictions $\mathbf{T}_{WL}$ and $\mathbf{T}_{GW}$, the \texttt{Combine} module employs the maximum weight matching, a traditional algorithm with enhanced weight setting strategies, and returns the matched node pairs $\mathcal{M}$. 
    Please note that without the \texttt{Combine} module, our model is end-to-end by using the alignment probability $\mathbf{T}_{GW}$ as prediction. 
    In Algorithm 1, we list the overall process of \texttt{CombAlign}.

Our model framework design is backed by theoretical motivations. In the following, we formally demonstrate that the \texttt{GRAFT} module (more specifically, feature propagation~\&~transformation) enhances the model's capability of ranking matched node pairs above the unmatched ones (Theorem~\ref{thm:T1} and Corollary~\ref{col:GNN}). 
Our choice of non-uniform marginals (i.e., the \texttt{NUM} module) is supported by Theorem~\ref{thm:NUM}, improving model expressiveness from a different perspective. 
Finally, we guarantee the \texttt{Combine} module (with Theorem~\ref{thm:combine}) to achieve the desired node matching properties, which also boosts alignment accuracy.

\begin{algorithm}
  \caption{The \texttt{CombAlign} Algorithm} \label{alg: CombAlign}
  \KwIn{Attributed graphs $\mathcal{G}_s(\mathbf{A}_s, \mathbf{X}_s)$ and $\mathcal{G}_t(\mathbf{A}_t, \mathbf{X}_t)$ }
  \KwOut{The set of predicted node matching $\mathcal{M}$}
  $\mathbf{T}_{WL} \leftarrow \texttt{WL}(\mathcal{G}_s, \mathcal{G}_t)$; \tcp{Algorithm~\ref{alg:WL}} \
  $\boldsymbol{\mu}, \boldsymbol{\nu} \leftarrow \texttt{NUM}(\mathbf{T}_{WL})$; \tcp{Algorithm~\ref{alg:NUM}} \
  $\mathbf{T}_{GW} \leftarrow \texttt{GRAFT}(\mathcal{G}_s, \mathcal{G}_t, \boldsymbol{\mu}, \boldsymbol{\nu})$; \tcp{Algorithm~\ref{alg:GRAFT}} \
  $\mathcal{M} \leftarrow \texttt{Combine}(\mathbf{T}_{WL}, \mathbf{T}_{GW})$; \tcp{Algorithm~\ref{alg:EL}} \
\Return $\mathcal{M}$;
\end{algorithm}

\vspace{-3mm}
\subsection{Improving Model's Discriminative Power} \label{sec:model-1} 
We elaborate on the designed modules by focusing on their discriminative power for graph alignment. 
We first introduce the \texttt{GRAFT} module followed by the \texttt{WL} and \texttt{NUM} modules as the latter two further enhance the expressiveness of the former. 
\subsubsection{The GW Learning (\texttt{GRAFT}) Module} \label{sec:model-GRAFT}
It comprises embedding-based feature propagation and transformation followed by the GW learning step with the idea of OT.

\noindent \textbf{Feature Propagation and Transformation}. 
As pointed out by~\cite{slotalign}, for OT-based approaches, the alignment quality heavily depends on the specific design of intra-graph costs $\mathbf{C}_s$ and $\mathbf{C}_t$. 
Representative methods~\cite{gwl, slotalign, dhot} integrate graph structures and node features as the following equation: 
\begin{equation}
\mathbf{R}_p = g_{prop}(\mathbf{A}_p) \mathbf{X}_p, p=s,t, 
\end{equation}
where $g_{prop}(\cdot)$ is a general function without learnable parameters. For example, we can set it to the standard propagation of GCN~\cite{gcn}, i.e., $ g_{prop}(\mathbf{A}) = {\tilde{\mathbf{D}}}^{-\frac{1}{2}} \tilde{\mathbf{A}} {\tilde{\mathbf{D}}}^{-\frac{1}{2}}$ with $\tilde{\mathbf{D}} = \tilde{\mathbf{A}} \mathbf{1}$ and $\tilde{\mathbf{A}} = \mathbf{A} + \mathbf{I}$.
We use $\mathbf{R}_s \in \mathbb{R}^{n_1 \times d}$ and $\mathbf{R}_t \in \mathbb{R}^{ n_2 \times d }$ to denote the result embedding matrices and refer to the process as \emph{feature propagation}\footnote{Note that~\cite{slotalign} uses parameter-free GNN to denote this process. In this paper, we propose strict definitions for both terms and formally analyze their differences.}. 

We observe that this step only enables the propagation of features between a node and its neighbors in a \emph{dimension-by-dimension} manner. 
However, interaction across different feature dimensions is not considered, which is also important, for example, when two feature dimensions share similar semantics. 
We put forward feature transformation, which is widely adopted by classical GNNs~\cite{gcn,gat,gin} and is formulated as
\begin{equation}
\mathbf{Z}_p = f_{GNN} (\mathbf{A}_p, \mathbf{X}_p, \mathbf{W}), p=s,t.
\end{equation}
Here, $f_{GNN}$ is a learnable function, and $\mathbf{W}$ denotes the learnable transformation matrices and is shared by two graphs, which enables feature interaction \emph{across dimensions}.
The following theorem demonstrates that, within the GW learning framework, the adoption of a single transformation layer for feature interaction enhances the discriminative power of the intra-graph cost matrices. 
Refer to all proofs in the Appendix.

\begin{theorem} \label{thm:T1}
We are given the graph structures $\mathbf{A}_s, \mathbf{A}_t$ and node features $\mathbf{X}_s, \mathbf{X}_t$ as input. 
Denote feature propagation as $\mathbf{R}_p = g(\mathbf{A}_p) \mathbf{X}_p, p = s,t$, where $g(\cdot)$ is a function without learnable parameters. 
Denote the additional linear transformation as $\mathbf{Z}_p = \mathbf{R}_p \mathbf{W}$, where $\mathbf{W} \in \mathbb{R}^{d \times d}$ is the learnable matrix. 
Assume that we set the intra-graph cost matrices as $\mathbf{C}'_p = \mathbf{R}_p\mathbf{R}^\intercal_p$ and $\mathbf{C}_p = \mathbf{Z}_p\mathbf{Z}^\intercal_p$ for $p=s,t$, respectively, and let $(u_i, v_k), (u_j, v_l) \in \mathcal{M}^* \text{ and } (u_{j'}, v_l) \notin \mathcal{M}^*$ where $\mathcal{M}^*$ is the ground truth. 
Then, there exists a case that $|\mathbf{C}'_s(i, j) - \mathbf{C}'_t(k, l)| = |\mathbf{C}'_s(i, j') - \mathbf{C}'_t(k, l)|$ and $|\mathbf{C}_s(i, j) - \mathbf{C}_t(k, l)| \neq |\mathbf{C}_s(i, j') - \mathbf{C}_t(k, l)|$. 
\end{theorem}

Moreover, we show that the discriminative power of our problem is closely related to that of GNNs (e.g.,~\cite{gin}). 
\begin{corollary} \label{col:GNN}
Given two GNNs $f_{GNN}$ and $f'_{GNN}$ for feature propagation~\&~transformation, where the expressive capability of $f_{GNN}$ is strictly more powerful than that of $f'_{GNN}$. 
    Denote by $\mathbf{C}$ and $\mathbf{C}'$ the intra-graph costs derived from $f_{GNN}$ and $f'_{GNN}$, respectively. 
    Then, our method has more discriminative power of ranking matched node pairs above the unmatched ones with $\mathbf{C}$.  
\end{corollary}

Motivated by the theoretical results, we employ three specific GNN models for feature propagation and transformation:
\begin{itemize}
\item Lightweight GCN~\cite{wu2019simplifying} (single linear transformation):
        \begin{equation}
       \hspace{-5mm} \mathbf{Z} = \text{Concat}(\mathbf{X}, {\mathbf{P}} \mathbf{X}, \ldots, {\mathbf{P}}^{K} \mathbf{X})  \mathbf{W}, 
        \mathbf{P} = {\tilde{\mathbf{D}}}^{-\frac{1}{2}} \tilde{\mathbf{A}} {\tilde{\mathbf{D}}}^{-\frac{1}{2}}.
        \end{equation}
    \item Graph Convolutional Network (GCN)~\cite{gcn}:
        \begin{equation}
            \mathbf{Z}^{(k+1)} = \sigma({\mathbf{P}}  \mathbf{Z}^{(k)} \mathbf{W}^{(k+1)}), k \in [0, K-1].
        \end{equation}
    \item Graph Isomorphism Network (GIN)~\cite{gin}:
        \begin{equation}
          \hspace{-5mm}  \mathbf{Z}^{(k+1)} = \mathrm{MLP}^{(k+1)} ( (1+\varepsilon^{(k+1)}) \mathbf{Z}^{(k)} + \mathbf{A} \mathbf{Z}^{(k)}   ), k \in [0, K-1].
        \end{equation}
\end{itemize}
For GCN and GIN, we use the skip connection to combine node embeddings of all layers: $\mathbf{Z}_p = \sum_{k=1}^{K}{\mathbf{Z}^{(k)}_p}, p=s,t$. 
The pseudocode is shown in Algorithm~\ref{alg: ft_prop_trans}.

\begin{algorithm}\label{alg: ft_prop_trans}
  \caption{Feature Propagation and Transformation (\texttt{FeatProp\&Trans})}  
  \KwIn{$\mathbf{A}, \mathbf{X}$, model parameter $\{ \mathbf{W}^{(1)}, \ldots, \mathbf{W}^{(K)} \}$} 
  \KwOut{Node representation $\mathbf{Z}$}
  
  $\mathbf{Z}^{(0)} \leftarrow \mathbf{X}$\;
  \For{$k = 1$ to $K$}
  {
    $\mathbf{Z}^{(k)} \leftarrow \sigma(g_{prop}(\mathbf{A}, \mathbf{Z}^{(k-1)}) \cdot \mathbf{W}^{(k)})$\;
  }
  \Return $\mathbf{Z} \leftarrow \sum_{k=1}^{K}{\mathbf{Z}^{(k)}}$\;
\end{algorithm}

\noindent \textbf{Gromov-Wasserstein Learning}.
Given the cost matrices as input, we follow the well-adopted approach in~\cite{gwl, slotalign, dhot} by employing the proximal point method~\cite{ppm,gwl} and reducing the learning problem to optimal transport~\cite{sinkhorn}. 
Denote by $L(\mathbf{T}, \boldsymbol{\beta}, \mathcal{W})$ the learning objective in Equation~\ref{eqn:GWD-2}, 
where $\boldsymbol{\beta} = (\boldsymbol{\beta}_s, \boldsymbol{\beta}_t)$ is the learnable coefficients to combine multiple terms for the intra-graph cost and $\mathcal{W}$ denotes the parameters in feature transformation.
We update $\mathbf{T}$ and $\boldsymbol{\Theta} = \{\boldsymbol{\beta}, \mathcal{W}\}$ by the proximal point method and gradient descent, respectively.
{\small
\begin{align} \label{eqn:F-update}
\boldsymbol{\Theta}^{(i+1)} &= \arg\min\left\{\nabla_{\boldsymbol{\Theta}} L(\mathbf{T}^{(i)}, \boldsymbol{\beta}^{(i)}, \boldsymbol{\Theta}^{(i)})^{\intercal} \boldsymbol{\Theta} + \frac{1}{2\tau_{\Theta}} \|\boldsymbol{\Theta} - \boldsymbol{\Theta}^{(i)}\|^{2}\right\}, \nonumber \\
\mathbf{T}^{(i+1)}&=\arg\min\left\{\nabla_{\mathbf{T}}L(\mathbf{T}^{(i)}, \boldsymbol{\beta}^{(i)}, \boldsymbol{\Theta}^{(i)})^\intercal \mathbf{T} + \frac{1}{\tau_{T}} \mathrm{KL}(\mathbf{T} \| \mathbf{T}^{(i)})\right\}.
\end{align}
}
Note that $\mathbf{KL}(\cdot||\cdot)$ is the Kullback-Leibler divergence, and $\tau_{\Theta}$ = $({\tau}_\beta, {\tau}_{W})$ is the learning rates for updating $\boldsymbol{\beta}$ and $\mathcal{W}$, respectively, while $\tau_T$ is the regularization coefficient in the Sinkhorn algorithm. 
It can be proved that with the simplified version of feature transformation, e.g., by adopting the lightweight GCN, 
the GW learning procedure theoretically guarantees the convergence result as in~\cite{gwl, slotalign} (Cf. Appendix).

\begin{algorithm}\label{GRAFT}
  \caption{Unsupervised Gromov-Wasserstein Learning with Feature Transformation (\texttt{GRAFT})} \label{alg:GRAFT}
  \KwIn{$\mathcal{G}_s$ and $\mathcal{G}_t$, marginal distributions $\boldsymbol{\mu}$ and $\boldsymbol{\nu}$
  }
  \KwOut{The OT-based alignment probability $\mathbf{T}_{GW}$}
  
  $\mathbf{T}_{GW} \leftarrow \boldsymbol{\mu} \boldsymbol{\nu}^{\intercal}$, $\boldsymbol{\beta}_s, \boldsymbol{\beta}_t \leftarrow (1,1,1)^{\intercal} $\;
 
  \For{$i = 1$ to $I$}
  {
  $\mathbf{Z}_s \leftarrow \texttt{FeatProp\&Trans}(\mathbf{A}_s, \mathbf{X}_s, \mathbf{W}^{(1, \ldots, K)})$\;
  $\mathbf{Z}_t \leftarrow \texttt{FeatProp\&Trans}(\mathbf{A}_t, \mathbf{X}_t, \mathbf{W}^{(1, \ldots, K)})$\;
    \tcp{intra-graph cost (Eq.~\ref{eqn:multi-view})}\
    $\mathbf{C}_p \leftarrow f_{\boldsymbol{\beta}}(\mathbf{A}_p, \mathbf{X}_p, \mathbf{Z}_p), p=s,t$\; 
    \tcp{inter-graph cost (Eq.~\ref{eqn:GWD})}\
    $\mathbf{C}_{gwd} \leftarrow f_{gwd}(\mathbf{C}_s, \mathbf{C}_t, \mathbf{T}_{GW})$\; 
    Minimize $\langle \mathbf{C}_{gwd}, \mathbf{T}_{GW} \rangle$ by updating $\boldsymbol{\Theta} = \{ \boldsymbol{\beta}, \mathbf{W}^{(1, \ldots, K)} \}$\;
    \tcp{the proximal point method}\ 
    Initialize $\mathbf{T}^{(0)} \leftarrow \mathbf{T}_{GW}$ and $\mathbf{a} \leftarrow \boldsymbol{\mu}$\;
  \For{$i' = 0$ to $I_{ot}-1$}{
    Set $\mathbf{G} \leftarrow \exp\left(-\frac{\mathbf{C}_{gwd}}{\tau_{T}} \right) \odot \mathbf{T}^{(i')}$\;
    \tcp{Sinkhorn-Knopp algorithm}
    \For{$j=1$ to $J$}{
        $\mathbf{b} \leftarrow \frac{\boldsymbol{\nu}}{\mathbf{G}^\intercal \mathbf{a}}$\;
        $\mathbf{a} \leftarrow \frac{\boldsymbol{\mu}}{\mathbf{G} \mathbf{b}}$\;
    }
    $\mathbf{T}^{(i'+1)} \leftarrow \text{Diag}(\mathbf{a}) \mathbf{G} \text{Diag}(\mathbf{b})$\;
  }
  $\mathbf{T}_{GW} \leftarrow \mathbf{T}^{(I_{ot})}$\;
  }
\Return $\mathbf{T}_{GW}$;
\end{algorithm}

The pseudocode of the \texttt{GRAFT} module is demonstrated in Algorithm~\ref{alg:GRAFT}. 
We first initialize the OT-based alignment matrix $\mathbf{T}_{GW}$ with the input marginals and set the learnable coefficients (Line 1). 
Then, for each iteration (Line 2), we invoke the feature propagation and transformation process to derive node representations (Lines 3-4), based on which the intra and inter-graph costs are computed (Lines 5-6). 
The GW learning procedure is employed (Lines 7-16), which takes the Sinkhorn algorithm~\cite{sinkhorn-knopp} as a subprocedure to iteratively update $\mathbf{T}_{GW}$, an $n_1 \times n_2$-sized alignment probability matrix. 

\subsubsection{The WL-then-Inner-Product (\texttt{WL}) Module} \label{sec:model-WL}
The \texttt{WL} module uses a GNN with non-learnable parameters (referred to as \emph{parameter-free GNNs}) to simulate the Weisfeiler-Lehman (WL) test~\cite{wl-test}. 
As demonstrated in Algorithm~\ref{alg:WL},
we use $K$ random matrices $\widehat{\mathbf{W}}^{(1)}, \ldots, \widehat{\mathbf{W}}^{(K)}$ to resemble the hash functions in WL test (Line 1). Taking node features as input, the algorithm conducts $K$ layers of graph convolution (Lines 2-3), similar to the GNN in Section~\ref{sec:model-GRAFT} but with non-updated parameters.
We denote by $\mathbf{H}_s$ and $\mathbf{H}_t$ the node embeddings in order to distinguish them from those obtained via a learnable GNN (i.e., $\mathbf{Z}_s$ and $\mathbf{Z}_t$). 
The embedding-based alignment probability $\mathbf{T}_{WL}$ is obtained by multiplying $\mathbf{H}_s$ and $\mathbf{H}_t$, followed by the normalization process (Line 4). 
We first adopt a ReLU activation and then divide each term by the summation of all terms in $\mathbf{T}_{WL}$. 
In this way, each element of the returned $\mathbf{T}_{WL}$ denotes the matching probability.

\begin{algorithm}
  \caption{WL-then-Inner-Product (\texttt{WL})}\label{alg:WL}
  \KwIn{$\mathcal{G}_s(\mathbf{A}_s, \mathbf{X}_s)$ and $\mathcal{G}_t(\mathbf{A}_t, \mathbf{X}_t)$}
  \KwOut{Embedding-based alignment probability $\mathbf{T}_{WL}$}
  
  Initialize non-learnable parameters $\widehat{\mathbf{W}}^{(1)}, \ldots, \widehat{\mathbf{W}}^{(K)}$\;  
  $\mathbf{H}_s \leftarrow \texttt{FeatProp\&Trans}(\mathbf{A}_s, \mathbf{X}_s, \widehat{\mathbf{W}}^{(1, \ldots, K)})$ \;
  $\mathbf{H}_t \leftarrow \texttt{FeatProp\&Trans}(\mathbf{A}_t, \mathbf{X}_t, \widehat{\mathbf{W}}^{(1, \ldots, K)})$\;
  $\mathbf{T}_{WL} \leftarrow \mathrm{Norm}(\mathbf{H}_s\mathbf{H}_t^\intercal$)\;
\Return $\mathbf{T}_{WL}$; 
\end{algorithm}

Although the prediction of \texttt{WL} is not as precise as \texttt{GRAFT} due to its simplicity, as we will show in the following, it serves as the prior knowledge for the OT-based component to boost matching accuracy.

\begin{figure}[t]
 \centering
    \begin{tabular}{cc}
        \hspace{0mm} \includegraphics[height=30mm]{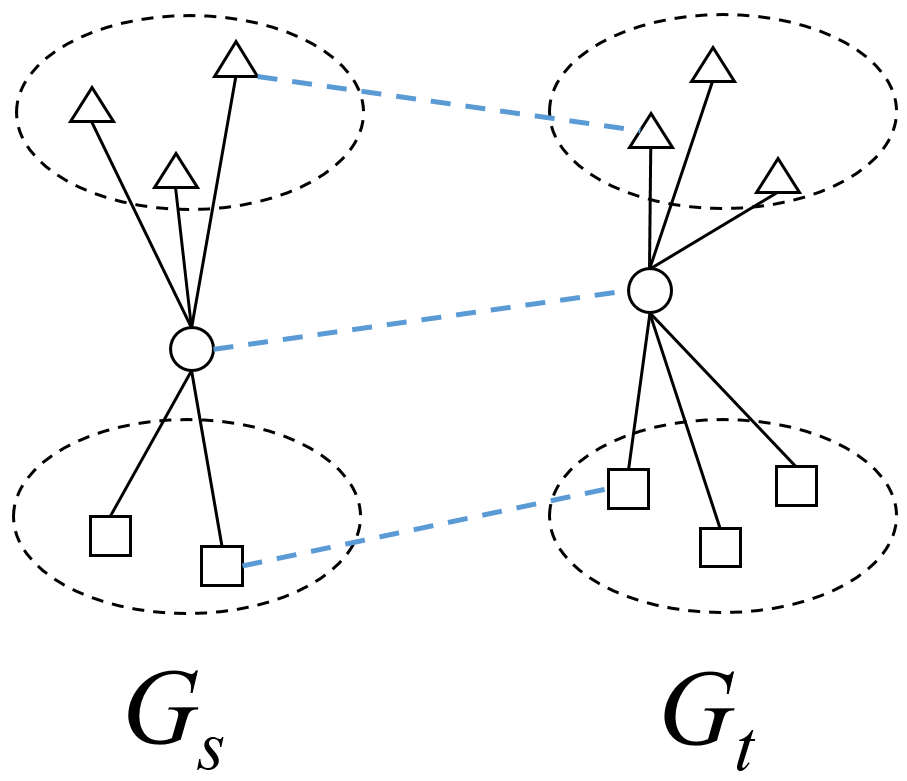} &
        \hspace{0mm} \includegraphics[height=27mm]{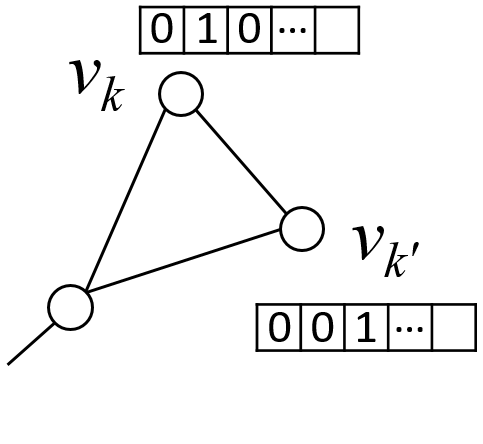} \\

       \hspace{0mm} (a) {An illustrative example} &
       \hspace{0mm} (b) {An inseparable case} \\
 \end{tabular}
\caption{Illustration of the necessity of non-uniform marginals.} \label{fig:NUM}
\end{figure}

\subsubsection{The Non-Uniform Marginal (\texttt{NUM}) Module} \label{sec:model-NUM}
It improves the \texttt{GRAFT} module by providing non-uniform marginals as prior knowledge. 
We observe that OT-based approaches take two marginal distributions on $\mathcal{V}_s$ and $\mathcal{V}_t$ as input, which is commonly assumed to be uniformly distributed~\cite{gwl, slotalign, dhot}. 
However, graphs are typical \emph{non-Euclidean} data where node distributions are not i.i.d. 
A motivating example is shown in Figure~\ref{fig:NUM}(a), where both $\mathcal{G}_s$ and $\mathcal{G}_t$ contain two clusters and a hub node, respectively. 
Suppose that the correspondence between nodes is consistent with their shapes.
If the pair of hubs is correctly aligned, it might eliminate a large amount of impossible alignments. 
We formally show that non-uniform marginals indeed enhance the expressive power of GW learning for graph alignment with the following theorem. 

\begin{theorem} \label{thm:NUM}
Consider the case that $(u_i, v_k) \in \mathcal{M}^*$ and $(u_i, v_{k'}) \notin \mathcal{M}^*$. 
For GW learning (e.g., \texttt{GRAFT}), under the mild assumption of the intra-graph cost, i.e.,
\begin{align}
&\mathbf{C}_s(i, i) = a, \forall u_i, \mathbf{C}_t(k, k) = b, \forall v_k, \\
&\mathbf{C}_s(i, j) = \mathbf{C}_s(j, i), \mathbf{C}_t(k, l) = \mathbf{C}_t(l, k), \forall u_i, u_j, v_k, v_l, \\
&\mathbf{C}_t(k, l) = \mathbf{C}_t(k', l), \forall v_l \in \mathcal{V}_t \backslash \{v_k, v_{k'}\},
\end{align}
with uniform marginals $\boldsymbol{\mu} = (1/n_1, \ldots, 1/n_1)^\intercal$ and $\boldsymbol{\nu} = (1/n_2, \ldots, 1/n_2)^\intercal$, the first iteration of the GW learning process with $\mathbf{T}^{(0)} = \boldsymbol{\mu} \boldsymbol{\nu}^\intercal$ cannot determine whether $u_i$ is matched to $v_k$ or $v_{k'}$. 
\end{theorem}


Note that apart from feature transformation, which improves model's discriminative power via the first part of each term of $\mathbf{C}_{gwd}(i, k)$, i.e., $\left | \mathbf{C}_s(i, j) - \mathbf{C}_t(k, l) \right |^2$, we offer another perspective by investigating the second term $\mathbf{T}(j, l)$. 
To solve this problem, the following heuristics are provided to set non-uniform marginals without extra computational costs. 

\noindent \textbf{A fixed prior inspired by WL}. 
We use the prior knowledge obtained from the WL test to set non-uniform marginals $\boldsymbol{\mu}$ and $\boldsymbol{\nu}$. 
Given the embedding-based alignment probability $\mathbf{T}_{WL}$, the marginals are computed by row/column-wise summation (shown in Algorithm~\ref{alg:NUM}).

\begin{algorithm}
\label{NUM}
  \caption{Non-Uniform Marginal (\texttt{NUM})}\label{alg:NUM}
  \KwIn{Embedding-based alignment probability $\mathbf{T}_{WL}$}
  \KwOut{Two marginal distributions $\boldsymbol{\mu}, \boldsymbol{\nu}$}
  $\boldsymbol{\mu} \leftarrow \mathbf{T}_{WL} \cdot \boldsymbol{1}_{n_2}, \boldsymbol{\nu} \leftarrow \mathbf{T}_{WL}^{\intercal} \cdot \boldsymbol{1}_{n_1}$\;
\Return $\boldsymbol{\mu}, \boldsymbol{\nu}$; 
\end{algorithm}

\noindent \textbf{Adaptive marginals during GW Learning}. 
Recall that $\mathbf{T}_{WL}$ is obtained by non-learnable embeddings $\mathbf{H}_s$ and $\mathbf{H}_t$, while the \texttt{GRAFT} module yields more informative node representation $\mathbf{Z}_s$ and $\mathbf{Z}_t$ through GW learning. 
Therefore, during each iteration (i.e., Line 2) of Algorithm~\ref{alg:GRAFT}, we use the learnable representations to compute an $n_1 \times n_2$-sized matrix following Algorithm~\ref{alg:WL}, and take it as the input of Algorithm~\ref{alg:NUM}. 
The marginals are used to initialize the proximal point method (i.e., Lines 8, 12~\&~13 of Algorithm~\ref{alg:GRAFT}) and are adjusted on the fly. 

\subsection{Ensuring Matching Properties (The \texttt{Combine} Module)} \label{sec:model-EL}

\noindent \textbf{Motivation}. 
As we have pointed out in Figure~\ref{fig:intro}(b)~\&~(c) of Section~\ref{sec:introduction}, pure learning-based methods~\cite{gwl,slotalign,dhot,gtcalign,hermanns2021grasp} fail to guarantee alignment properties such as one-to-one matching and mutual alignment, albeit that they can achieve better practical accuracy than the traditional approaches~\cite{feizi2019spectral,karakasis2021joint,peng2010new}. 
In contrast, existing studies (e.g.,~\cite{kolavr2012graphalignment}) have demonstrated real-world scenarios in which one-to-one node matching needs to be ensured. 
The inconsistency in mutual alignment means that for a predicted matching $(u, v)$ taking from the row/column-wise maximum of the alignment probabilities, we have $\argmax_y \mathbf{T}(u, y) = v$ and $\argmax_x \mathbf{T}(x, v) \neq u$. 
This contradicts the essence of graph alignment where aligned nodes are paired up.
We demonstrate in Figure~\ref{fig:missingrate} that both problems are prevalent on real-world datasets. 
Given the output of an algorithm, we calculate the percentages of one-to-many predictions and inconsistencies in terms of mutual alignment, and find that a significant ratio of the predictions violates these alignment properties.

\begin{figure}
    \centering
    {\includegraphics[height=5mm]{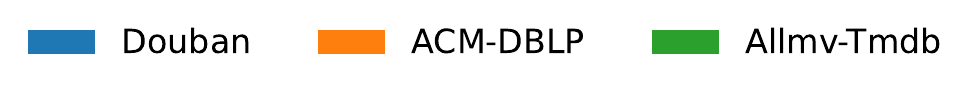}}
    {\includegraphics[height=5mm]{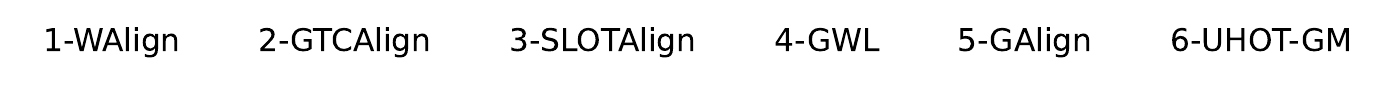}}
    \begin{tabular}{cc} 
        \hspace{-4mm}\includegraphics[height=35mm]{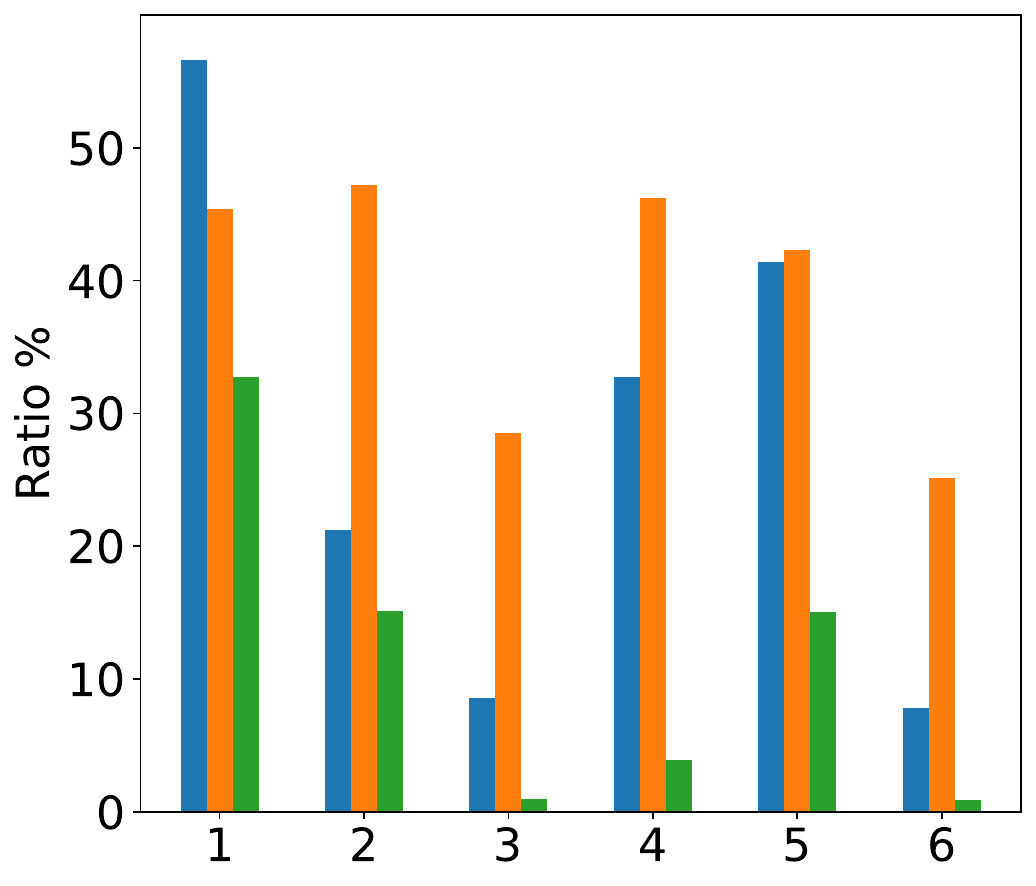} & \hspace{-4mm}\includegraphics[height=35mm]{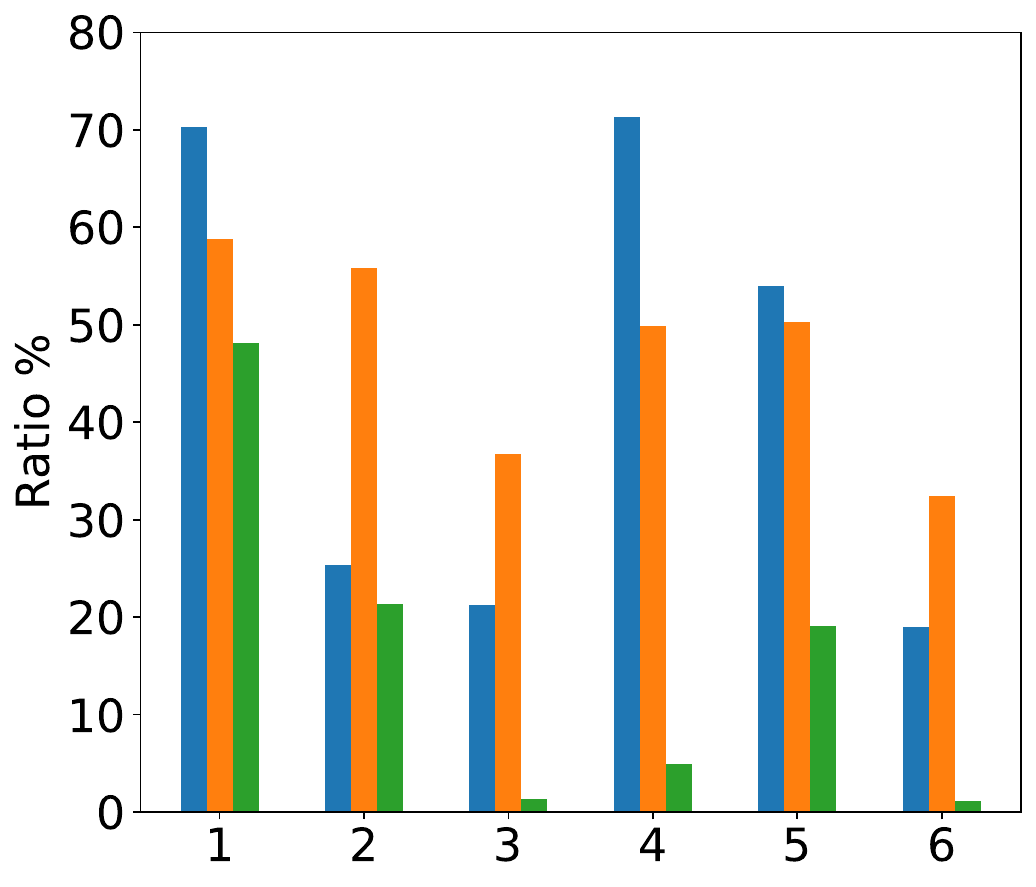}\\

        (a) One-to-many prediction &
        (b) Inconsistent alignment
    \end{tabular}
    \caption{ Percentage of one-to-many predictions and inconsistencies in terms of mutual alignment for representative baselines.
    }
    \label{fig:missingrate}
\end{figure}

We observe that \emph{maximum weight matching} is a feasible solution to this problem, which has been adopted by tradtional approaches. 
It takes a bipartite graph with weighted edges as input and finds a set of edges with maximum total weight on the condition that selected edges do not share any endpoint (i.e., node). 
By constructing a bipartite graph from the node sets, finding the maximum weight matching is equivalent to graph alignment with the additional benefit of holding the matching properties. 
The remaining challenge is to guarantee the accuracy of the alignment. 
Intuitively, the edge weights determine the alignment accuracy and should be close to the ground truth alignment. 
Our solution, referred to as the \texttt{Combine} module, tackles this problem by effectively combining both embedding and OT-based alignment probabilities. 
Specifically, it includes the following steps (Cf. Figure~\ref{fig:framework}).

\noindent \textbf{Bipartite Graph Construction}. 
To construct the bipartite graph $\mathcal{G}_b$, we set $\mathcal{V}_b = \mathcal{V}_s \cup \mathcal{V}_t$.
Instead of building edges for all node pairs, we observe that in most cases, the ground truth is contained in the top-$r$ predictions. 
As the \texttt{GRAFT} module generally achieves better performance, for each node $u_i \in \mathcal{G}_s$, we connect it to the nodes in $\mathcal{G}_t$ with top-$r$ largest $\mathbf{T}_{GW}(i, \cdot)$.
This step reduces the edge number from $n_1n_2$ to $rn_1$ and significantly improves practical efficiency.

\noindent \textbf{Edge Weight Computation}. 
We adopt the idea of \emph{ensemble learning} (more specifically, \emph{stacking}) and see both embedding and OT-based predictions as base learners.
We employ a simple yet effective {parameter-free} strategy to take the best of both worlds and believe that it tends to obtain more accurate predictions, i.e., the edge weight being close to the ground truth alignment.
We consider the following concrete forms of edge weights and use the ``both confident'' setting by default:
\begin{itemize} [leftmargin=0.35cm, itemindent = 0cm]
\item ``Both confident’’: $ g_{EL}(\mathbf{T}_{WL}, \mathbf{T}_{GW}) = \mathbf{T}_{WL} \odot \mathbf{T}_{GW}$,

\item ``Simple average’’: $ g_{EL}(\mathbf{T}_{WL}, \mathbf{T}_{GW}) = (\mathbf{T}_{WL} + \mathbf{T}_{GW})/2$.

\end{itemize}

\noindent \textbf{Maximum Weight Matching}.
Note that any maximum weight matching algorithm with real-value weights (e.g., the modified Jonker-Volgenant algorithm (\texttt{MJV})~\cite{crouse2016implementing}) can be invoked to solve the problem. 
The pseudocode of the \texttt{Combine} module is shown in Algorithm~\ref{alg:EL}, which returns a set of (one-to-one) matched node pairs instead of the alignment probabilities.

\begin{theorem} \label{thm:combine}
With the \texttt{Combine} module, our algorithm predicts a set of matched node pairs with desired properties, including one-to-one matching and mutual alignment. 
\end{theorem}

\begin{algorithm}
  \caption{The \texttt{Combine} Module}\label{alg:EL} 
  \KwIn{Embedding~\&~OT-based outputs, $\mathbf{T}_{WL}$, $\mathbf{T}_{GW}$}
  \KwOut{The combined prediction $\mathcal{M}$}
  Construct bipartite graph $\mathcal{G}_b$ with $\mathcal{V}_b \leftarrow \mathcal{V}_s \cup \mathcal{V}_t$\;
  Construct $\mathcal{E}_b$ by row-wise top-$r$ prediction of $\mathbf{T}_{GW}$\;
  $\mathcal{E}_b.weight \leftarrow g_{EL}(\mathbf{T}_{WL}, \mathbf{T}_{GW})$\;
  Compute the maximum weight matching $\mathcal{M}$ for $\mathcal{G}_b$\;
\Return $\mathcal{M}$; 
\end{algorithm}

\subsection{Algorithm Analysis and Optimization}

\subsubsection{Complexity Analysis}
We discuss the complexity of \texttt{CombAlign} by analyzing each module in brief.  
Denote by $n$, $m$ the total number of nodes and edges in the graph, and $d$ the hidden dimension. 
The \texttt{WL} module takes $O(K(md + nd^2) + n^2d)$ time, where the first term follows classical message-passing GNNs and the second term correpsonds to the multiplication of two $n \times d$-sized embedding matrices. 
The \texttt{NUM} module only incurs $O(n^2)$ time for the matrix-vector multiplication. 
The \texttt{GRAFT} module has the same asymptotic complexity with \texttt{SLOTAlign}~\cite{slotalign} since the additional feature transformation only needs $O(Kd^2)$ time. 
To be more specific, for both the \texttt{GRAFT} module and \texttt{SLOTAlign}, the cost computation of GW learning is $O(n^3)$ per round according to~\cite{gwl}, with an inner-loop of $I_{ot}$ rounds for the Proximal Point Method. 
For the \texttt{Combine} module, the bipartite graph construction step takes linear time, and the complexity is dominated by maximum weight matching. 
Although classical algorithms~\cite{km, munkres1957algorithms, crouse2016implementing} typically need $O(n^3)$ time, note that our bipartite graph only contains $O(rn_1)$ edges, which corresponds to the \emph{sparse} linear assignment problem with $O(n^2r)$ complexity~\cite{cui2016solving}. 
This process can be further improved with faster matching algorithms~\cite{schwartz1994computational,buvs2009towards}.
In sum, the asymptotic complexity is $O(I(K(md+nd^2) + n^2d + I_{ot}n^3))$, which is in the same order as other optimal transport-based alignment methods~\cite{gwl, fusegwd,dhot,slotalign}. 
In comparison, we also demonstrate the time complexity of two representative methods in Table~\ref{tab:complexity}.
Note that both embedding and OT-based methods need to iteratively update the alignment probabilities, thus we use $I$ to represent the number of iterations.

\begin{table}[t]
\centering
\small
\caption{Comparison of model complexity with state of the art.}
\label{tab:complexity}
\begin{adjustbox}{max width=\columnwidth}
\begin{tabular}{|c|c|}
\hline
Method & Time Complexity \\
\hline
\texttt{GTCAlign}~\cite{gtcalign} (Emb.) & $O\!\left(IK(md + n^2 d)\right) \approx O(n^2)$ \\
\hline
\texttt{SLOTAlign}~\cite{slotalign} (OT) & $O\!\left(I\!\left(K(md + n d^2) + n^2 d + I_{ot} n^3\right)\right) \approx O(n^3)$ \\
\hline
\texttt{CombAlign} & $O\!\left(I\!\left(K(md + n d^2) + n^2 d + I_{ot} n^3\right)\right) \approx O(n^3)$ \\
\hline
\texttt{CombAlign(Opt)} & $O\!\left(I\!\left(K(md + n d^2) + n^2\!\left(d + I_{ot}\log n\right)\right)\right) \approx O(n^2)$ \\
\hline
\end{tabular}
\end{adjustbox}
\end{table}

\subsubsection{Optimizations}
According to our analysis, the asymptotic complexity of \texttt{CombAlign} (and other OT-based approaches) can be simplified to $O(n^3)$ since $n$ dominates other inputs (i.e., $I, K, $ and $d$). 
Unfortunately, a time complexity gap is observed between embedding-based solutions~\cite{walign, galign, gtcalign} and \texttt{CombAlign}, while the former only takes roughly $O(n^2)$ time. 
The complexity bottleneck of OT arises from the GWD cost computation (i.e., $\mathbf{C}_{gwd}$). 
According to~\cite{gwl}, it can be reformalized as $\mathbf{C}_{gwd} = \mathbf{C}_s^2 \boldsymbol{\mu} \mathbf{1}^\intercal_{n_2} + \mathbf{1}_{n_1} \boldsymbol{\nu}^\intercal \mathbf{C}_t^2 - 2 \mathbf{C}_s  \mathbf{T} \mathbf{C}_t^\intercal$, reducing the computational complexity from $O(n^4)$ to $O(n^3)$. (Note that the square operation is element-wise.) 
Since only the third term has a complexity of $O(n^3)$, we simplify the intra-graph cost design by only utilizing the adjacency matrix, which leads to a $\mathbf{A}_s  \mathbf{T} \mathbf{A}_t^\intercal$ term. 
For real-world graphs, the adjacency matrix is sparse, namely, the number of edges $m$ is bounded by $O(n\mathrm{log}n)$~\cite{wang2017fora}, and it holds that $m/n \ll d$ in practice (Cf. Section~\ref{sec:experiments}). 
To this end, multiplying two sparse matrices with a dense one can be fulfilled in $O(nm) = O(n\mathrm{log}n)$ time. 
We also employ powerful embedding learning models (e.g.,~\cite{wu2023sgformer}) in the \texttt{GRAFT} module to alleviate the degradation of model expressiveness. 
As shown in Table~\ref{tab:complexity}, our optimized model successfully closes the complexity gap between two categories of learning-based approaches.  

For large-scale graphs, such as those with millions of entities, our algorithm can be easily integrated with a divide-and-conquer strategy~\cite{guo2025s3ga}: first partitioning the source and target graphs into
smaller subgraphs, then performing alignment among the subgraphs, and finally aggregating the matching results (Cf. Section~\ref{sec:efficiency-comparison}). 

\section{Experiments}\label{sec:experiments}

We conduct a comprehensive experimental evaluation to answer the following key research questions:
\begin{itemize}
    \item Q1. Does \texttt{CombAlign} outperform the state of the art? 
    \item Q2. How effective is each component of \texttt{CombAlign}?
    \item Q3. Is \texttt{CombAlign} capable of holding matching properties such as one-to-one matching and mutual alignment?    
    \item Q4. How efficient and scalable is \texttt{CombAlign}?
\end{itemize}

\subsection{Experimental Settings}
\noindent \textbf{Datasets}. 
We mainly use the well-adopted datasets for unsupervised graph alignment~\cite{slotalign, dhot, gtcalign, gwl, walign}, including three real-world datasets Douban Online-Offline~\cite{zhang2016final}, ACM-DBLP~\cite{zhang2018attributed}, Allmovie-Imdb~\cite{galign}, and five synthetic ones, i.e., Cora~\cite{sen2008collective}, Citeseer~\cite{sen2008collective}, PPI~\cite{zitnik2017predicting}, and two larger datasets CS and Physics~\cite{shchur2018pitfalls}. 
The statistics of the datasets are listed in Table~\ref{table:dataset-statistics}, including the numbers of nodes ($n_1$ and $n_2$) and edges ($m_1$ and $m_2$), and known anchors only for evaluation. 
It is noteworthy that all ground-truth provided in Table~\ref{table:dataset-statistics} are one-to-one node correspondences.  
We adopt two datasets with one million nodes~\cite{guo2025s3ga} for scalability analysis.

\begin{table} 
\centering
  \caption{Datasets and their statistics. } 
  \label{table:dataset-statistics}
  \begin{tabular}{c|ccp{0.8cm}p{0.9cm}}
    \toprule
    Dataset & $n_1, n_2$ & $m_1, m_2$ & Features & Anchors \\
    \midrule
    \multirow{2}{*}{Douban~\cite{zhang2016final}} & $1,118$ & $3,022$ & \multirow{2}{*}{$538$} & \multirow{2}{*}{$1,118$} \\
                                           & $3,906$ & $16,328$  \\
                                           \midrule
    \multirow{2}{*}{ACM-DBLP~\cite{zhang2018attributed}} & $9,872$ & $39,561$ & \multirow{2}{*}{$17$} & \multirow{2}{*}{$6,325$} \\
                                           & $9,916$ & $44,808$  \\
                                           \midrule
    \multirow{2}{*}{Allmovie-Imdb~\cite{galign}} & $6,011$ & $124,709$ & \multirow{2}{*}{$14$} & \multirow{2}{*}{$5,174$} \\
                                           & $5,713$ & $119,073$  \\
                                           \midrule
    
   \multirow{2}{*}{Cora~\cite{sen2008collective}} & $2,708$ & $5,028$ & \multirow{2}{*}{$1,433$} & \multirow{2}{*}{$2,708$} \\
                                       & $2,708$ & $5,028$  \\
                                       \midrule        
   \multirow{2}{*}{Citeseer~\cite{sen2008collective}} & $3,327$ & $4,732$ & \multirow{2}{*}{$3,703$} & \multirow{2}{*}{$3,327$} \\
                                       & $3,327$ & $4,732$  \\
                                          \midrule
   \multirow{2}{*}{PPI~\cite{zitnik2017predicting}} & $1,767$ & $16,159$ & \multirow{2}{*}{$50$} & \multirow{2}{*}{$1,767$} \\
       & $1,767$ & $16,159$  \\
        \midrule
    \multirow{2}{*}{CS~\cite{shchur2018pitfalls}} & $18,333$ & $114,652$ & \multirow{2}{*}{$6,805$} & \multirow{2}{*}{$18,333$} \\
       & $18,333$ & $147,410$  \\
        \midrule
       \multirow{2}{*}{Physics~\cite{shchur2018pitfalls}} & $34,493$ & $347,147$ & \multirow{2}{*}{$8,415$} & \multirow{2}{*}{$34,493$} \\
       & $34,493$ & $ 446,332$  \\
    \bottomrule
\end{tabular}
\end{table}

\noindent \textbf{Baselines}. 
We compare \texttt{CombAlign} with traditional solutions and representative embedding and OT-based models. 
For \emph{traditional approaches}, we include \texttt{KNN} and use parameter-free GNNs to obtain node embeddings, based on which the top-$k$ similar node pairs across graphs are matched. 
We also adopt the modified version of the Jonker-Volgenant algorithm~\cite{crouse2016implementing} (\texttt{MJV}) for the linear assignment problem, which returns one-to-one alignment. 
We compare with representative \emph{embedding-based models}, including \texttt{WAlign}~\cite{walign}, \texttt{GAlign}~\cite{galign} and \texttt{GTCalign}~\cite{gtcalign}, as well as \emph{OT-based models}, i.e., \texttt{GWL}~\cite{gwl}, \texttt{SLOTAlign}~\cite{slotalign}, \texttt{UHOT-GM}~\cite{dhot} and \texttt{HLOT}~\cite{HLOT}. 
(Refer to Related Work for more details.) 
We omit other learning-based methods (e.g.,~\cite{fusegwd,hermanns2021grasp}) as they have been surpassed by the above baselines according to existing literature~\cite{walign, gtcalign, gwl, slotalign}.

\noindent \textbf{Evaluation Metrics}. 
Following previous studies~\cite{slotalign,dhot,gtcalign,walign}, we adopt Hits@$k$ with $k = \{1, 5, 10\}$ and Mean Average Precision (MAP) to evaluate the model performance. 
Given the ground truth $\mathcal{M}^*$, for each $(u, v) \in \mathcal{M}^*$ with $u \in \mathcal{G}_s$ and $v \in \mathcal{G}_t$, we check if $v$ belongs to the top-$k$ predictions ordered by alignment probability, denoted as $S_k(u)$. 
Then, Hits@$k$ is computed as follows:
\begin{equation} \label{eqn:Hits-k}
\text{Hits@}k = \frac {\sum_{(u, v) \in \mathcal{M}^*} {\mathds{1}[v \in S_k(u)]} }{ |\mathcal{M}^*| },
\end{equation}
where $\mathds{1}[\cdot]$ is the indicator function which equals 1 if the condition holds. 
MAP is employed to evaluate the accuracy of the predicted node rankings:
\begin{equation}
    \text{MAP} = \frac{\sum_{(u_i, v_k) \in {\mathcal{M}^*}}
    \frac{1}{{Rank}(u_i, v_k)}}{\left|{\mathcal{M}^*}\right|},
\end{equation}
where $Rank(u_i, v_k)$ denotes the ranking of $v_k$ according to $\mathbf{T}(i, \cdot)$ by descending order.

\noindent \textbf{Implementation Details}. 
For all baselines, we obtain the source code from the authors and make sure that all
baseline models are tuned to the best performance.  For datasets (e.g., Allmovie-Imdb) where some baselines do
not use them in the original paper, we perform the hyperparameter search to obtain the best results. Our model is implemented based on PyTorch 1.12.1 and PyTorch Geometric. Specifically, for \texttt{CombAlign}, we set the feature dimension $d$ as 32 and the graph convolution layers $K$ as 3.  
The learning rates $\tau_{\beta}$ and $\tau_{W}$ for updating $\beta$ and $\mathcal{W}$ are set to 1 and 0.01, respectively. 
We conduct experiments on a high-performance computing server with a GeForce RTX 4090 GPU equipped with 24 GB of memory. 

\begin{table*}[t]
    \centering
        \begin{threeparttable}
            \caption{Comparison of model performance on six datasets.}
            \begin{tabular}{cc|cccccccccc}
                \toprule
                Datasets &Metrics&\texttt{KNN}& \texttt{GAlign}& \texttt{WAlign} & \texttt{GTCAlign} & \texttt{GWL} & \texttt{SLOTAlign} & \texttt{UHOT-GM} & ~\texttt{HLOT}&\texttt{CombAlign w/o C} \\
                \midrule
                \multirow{4}{*}{Douban } & Hits@1&27.82 & 45.26 & 39.45&\underline{61.79} & 3.29  & {51.43} & {60.23}   & 57.26& \textbf{68.52}\\
                &Hits@5&45.53& 67.71 & 62.35& \underline{76.83} & 8.32  & {73.43} & 71.36  &79.77&  \textbf{87.84} \\
                &Hits@10&52.68 & 78.00 & 71.47& \underline{82.29} & 9.93  & {77.73} & 76.91 &84.09 &  \textbf{91.41}\\
                & MAP&36.08 &56.32 &46.22 &\underline{69.77} &5.79 & 61.29& 67.35 &63.45 &\textbf{77.08} \\
                \midrule
                \multirow{4}{*}{ACM-DBLP} & Hits@1&36.35 & \underline{70.20}&63.43& 60.92 & 56.36   & {66.04} & 69.89  &65.86 & \textbf{72.18} \\
                &Hits@5&66.83 &\underline{87.23}& 83.18& 75.60 & 77.09  & {85.84} & 87.12 &84.74 &  \textbf{88.98} \\
                &Hits@10&76.22&\underline{91.36} & 86.58& 79.97 & 82.18   & {87.76} & 90.65 &88.82 & \textbf{92.63} \\
                & MAP&50.11 &\underline{77.49}& 70.76&67.67&64.82  &73.76 &77.18  &72.89 &\textbf{79.55} \\
                \midrule
                \multirow{4}{*}{Allmovie-Imdb} & Hits@1&32.39 & {82.14}& 52.61  & 84.73&87.82  & {90.60} & \underline{91.73} & 90.88&\textbf{96.25} \\
                &Hits@5& 51.57&86.35 & 70.91& 89.89 & 92.31  & {92.75} & \underline{94.36}  &  93.01&\textbf{97.66} \\
                &Hits@10& 58.79&90.03 & 76.52 & 91.32& 92.83 & {93.14} & \underline{94.96} & 93.74& \textbf{97.89} \\
                & MAP&41.50&84.96  &61.17&87.12 &89.64 &{91.61} &\underline{92.74}  &91.83&\textbf{97.31} \\
                \midrule
                \multirow{4}{*}{Cora} &Hits@1& 95.01& 99.45 & 98.45& 99.35 & 86.19  & \underline{99.48} & \underline{99.48} & 99.27& \textbf{99.56} \\
                & Hits@5 &\textbf{100}&\textbf{100} & \textbf{100} & \textbf{100}& 93.61  &\textbf{100} & \textbf{100} & \textbf{100} & \textbf{100} \\
                & Hits@10 &\textbf{100}&\textbf{100}& \textbf{100}& \textbf{100} & 94.57  &\textbf {100} & \textbf{100} & \textbf{100} & \textbf{100} \\
                & MAP&98.66 &99.69&99.18&99.69 &89.71  &
                {99.71} &\underline{99.72}  &99.63 &\textbf{99.75} \\
                
                \midrule
                \multirow{4}{*}{Citeseer} & Hits@1 &89.72& \underline{99.73} & 97.81& 99.68  & 57.05  & 99.25 & 99.47& 99.23& \textbf{99.82} \\
                & Hits@5&\textbf{100}& \textbf{100} &\textbf{100}& \textbf{100} & 65.04  & \textbf{100} & \textbf{100}  & \textbf{100}& \textbf{100} \\
                & Hits@10&\textbf{100}& \textbf{100} & \textbf{100} & \textbf{100}& 65.95  & \textbf{100} & \textbf{100}  & \textbf{100} & \textbf{100}\\
                & MAP&94.91&99.84 &98.88&\underline{99.89} &61.31  &99.62 &99.69 &99.79 &\textbf{99.91} \\

                \midrule
                \multirow{4}{*}{PPI} & Hits@1&84.96& 89.20 & 88.51& 89.25 & 86.76  & {89.30} & \underline{89.33}   &89.26 &\textbf{89.70} \\
                  & Hits@5&89.10& 90.64 & \underline{93.10} & 92.81& 88.06  & {92.53} & 92.97 &92.89 & \textbf{93.15} \\
                & Hits@10&92.17 &{94.16}& \underline{94.17}& 94.07 & 88.62 &   {93.49} & 93.52  &93.46& \textbf{94.85} \\
                & MAP&87.65&90.72 & 89.02&\underline{90.80}&87.74  &90.76 &90.81   &90.65&\textbf{91.12} \\
                
                \bottomrule
            \end{tabular}

            \label{table:result}
        \end{threeparttable}
\end{table*}

\begin{table}
\centering
\caption{ Comparison of Hits@1 with one-to-one matching constraint.}
\begin{tabular}{|p{20mm}|ccc|}
\hline
 Datasets& Douban & ACM-DBLP &Allmovie-Imdb \\
 \hline
\texttt{KNN} &23.88 &31.11  & 28.79    \\
\hline
\texttt{MJV}  &{31.03} &  {56.89} &  {35.14} \\
\hline
\texttt{GAlign}  & 20.84 &   {66.18} & 80.28 \\
\hline
\texttt{WAlign}  & {15.56} & {60.33}  &{51.29} \\
\hline
\texttt{GTCAlign}  & {57.15} & {56.96} & {83.91} \\
\hline
\texttt{GWL} & {3.22} & {52.09} & {87.22}\\
\hline
\texttt{SLOTAlign}  &{49.19} &  {64.32} &  {90.31} \\
\hline
\texttt{UHOT-GM}  &\underline{58.31} &  \underline{67.78} &  \underline{91.26} \\
\hline
\texttt{HLOT}  &{54.69} &  {64.14} &  {90.37} \\
\hline
\texttt{CombAlign }  & \textbf{70.75} & \textbf{74.19} & \textbf{96.57} \\
 \hline
\end{tabular}
\label{tab: one to one}
\end{table}

\begin{table}
\centering
\caption{The performance (Hits@1) of \texttt{GRAFT} with different GNNs.}
\begin{tabular}{|p{22mm}|ccc|}
\hline
Hits@1 & Douban & ACM-DBLP &Allmovie-Imdb \\
\hline
 w/ LGCN & {57.07} & {68.87}  &{93.39}  \\
\hline
 w/ GCN & 61.45 &  69.68 & 94.57 \\
\hline
 w/ GIN & {63.15} & {71.64} & {95.61} \\
 \hline
 w/ SGFormer & {65.38} & {72.73} & {96.02} \\
 \hline
\end{tabular}
\label{tab: diff-gnn}
\end{table}

\subsection{Experimental Results}
\subsubsection{Model Performance} 
\noindent \texttt{CombAlign} \textbf{w/o the} \texttt{Combine} \textbf{module} (\texttt{CombAlign w/o C}). 
We first evaluate the prediction accuracy following previous learning-based settings~\cite{slotalign, gtcalign}, specifically, the output is formulated as the alignment probability matrix. 
As shown in Table~\ref{table:result}, we use the bold font to highlight the best results and underline the second-best values\footnote{The reported figures are the average of five independent rounds, and we omit the standard errors following most existing work because the prediction accuracy is highly concentrated for this problem.}. 
As for the baselines, note that \texttt{GAlign} achieves the best performance on ACM-DBLP while \texttt{UHOT-GM} has the highest accuracy on Allmovie-Imdb. 
For the Douban dataset, \texttt{GTCAlign} and \texttt{UHOT-GM} are the two best baselines. 
It can be concluded that neither OT-based methods nor embedding solutions can consistently outperform the other. 
For the embedding-based algorithms, we note that there is no exist a clear winner between \texttt{GAlign} and \texttt{GTCAlign} on the real-world datasets, demonstrating the limitation of the embedding-only heuristics in the unsupervised setting. 
Nonetheless, both of them outperform \texttt{WAlign}, while the training process of the latter encounters significant oscillations, partially due to the challenges in training GANs~\cite{wgan-improved}. 
For the OT-based methods, in general, a more sophisticated design of intra-graph costs with multiple terms (i.e., \texttt{SLOTAlign}) and cross-modal comparison (i.e., \texttt{UHOT-GM}) result in better practical accuracy. 

For the \texttt{CombAlign w/o C} algorithm, in our experiments, we adopt GCN~\cite{gcn} as the default GNN module in \texttt{GRAFT} and \texttt{WL}. 
It consistently demonstrates superior performance across all evaluation metrics on three real-world datasets, validating the effectiveness of our model. Notably, on the Douban and Allmovie-Imdb datasets, it achieves significant improvements in Hits@1, with relative gains of 10.89\% and 4.93\%, respectively. 
For the synthetic datasets, \texttt{CombAlign w/o C} also achieves the highest performance. 

\noindent \texttt{CombAlign} \textbf{(w/ the} \texttt{Combine} \textbf{module)}. 
As presented in Table~\ref{tab: one to one}, we evaluate \texttt{CombAlign} against the baselines by forcing the one-to-one matching constraint. 
Among the evaluated methods, only \texttt{MJV} and \texttt{CombAlign} inherently ensure one-to-one matching, as they directly output matched node pairs. 
For other methods generating the alignment probability matrices, we eliminate all one-to-many predictions by retaining the node pair with the highest probability. (Also note that mutual alignment is non-trivial to obtain for these methods.) 
It is observed that these methods inevitably suffer from performance degradation in this setting (Cf. Table~\ref{table:result}), and the phenomenon is significant on Douban and ACM-DBLP datasets. 

In comparison, \texttt{CombAlign} consistently outperforms all other methods. 
Notably, the substantial improvement of \texttt{CombAlign} over \texttt{MJV} underscores the critical role of weight design and demonstrates the effectiveness of the proposed ensemble learning strategy. 
Compared to \texttt{CombAlign w/o C}, \texttt{CombAlign} not only ensures one-to-one matching (and mutual alignment) but also significantly enhances performance.

\begin{figure}
    \centering
    {\includegraphics[height=7.5mm]{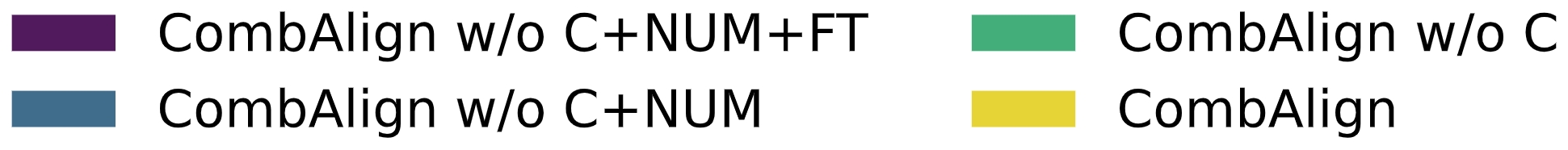}}
    \begin{tabular}{ccc} 
        \hspace{-4mm}\includegraphics[height=40mm]{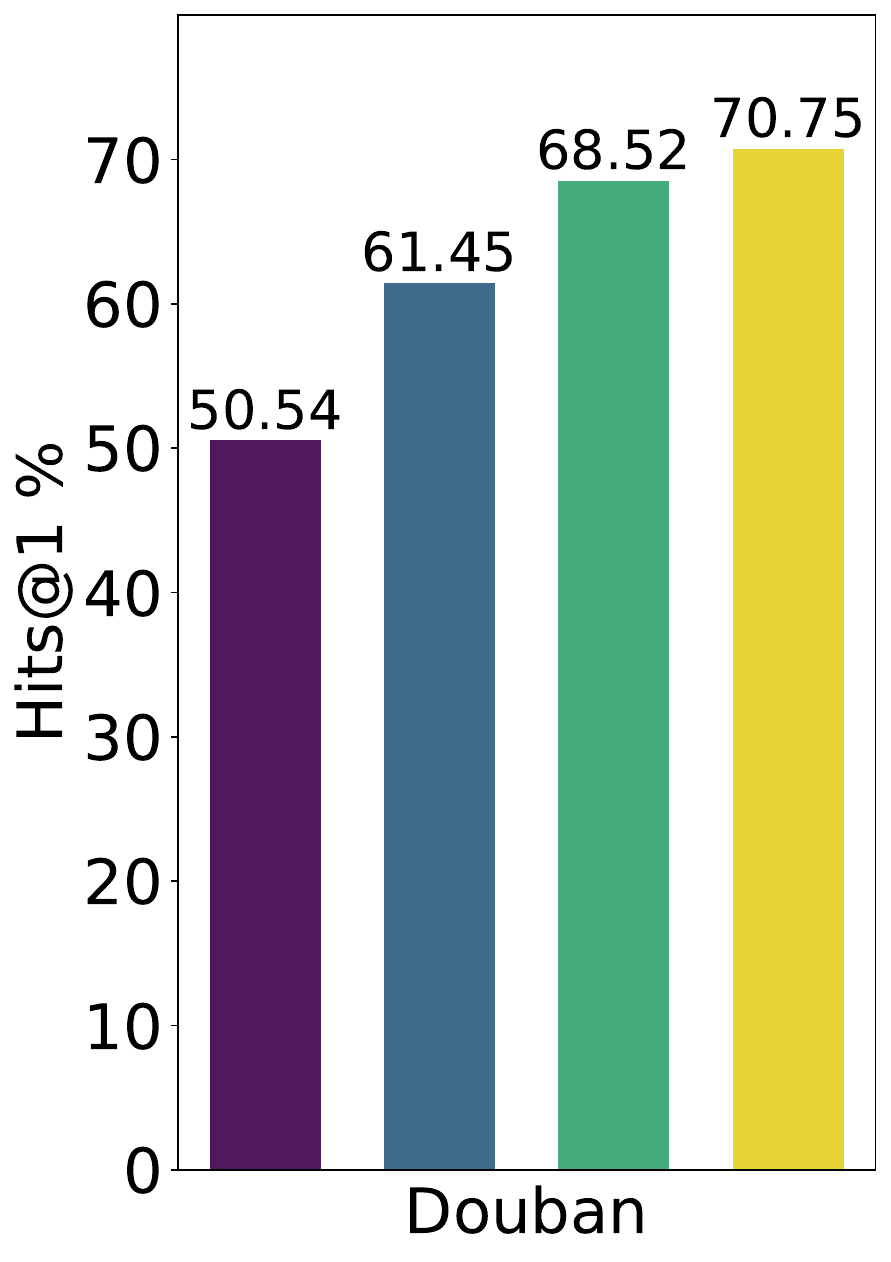} & 
         \hspace{-4mm}\includegraphics[height=40mm]{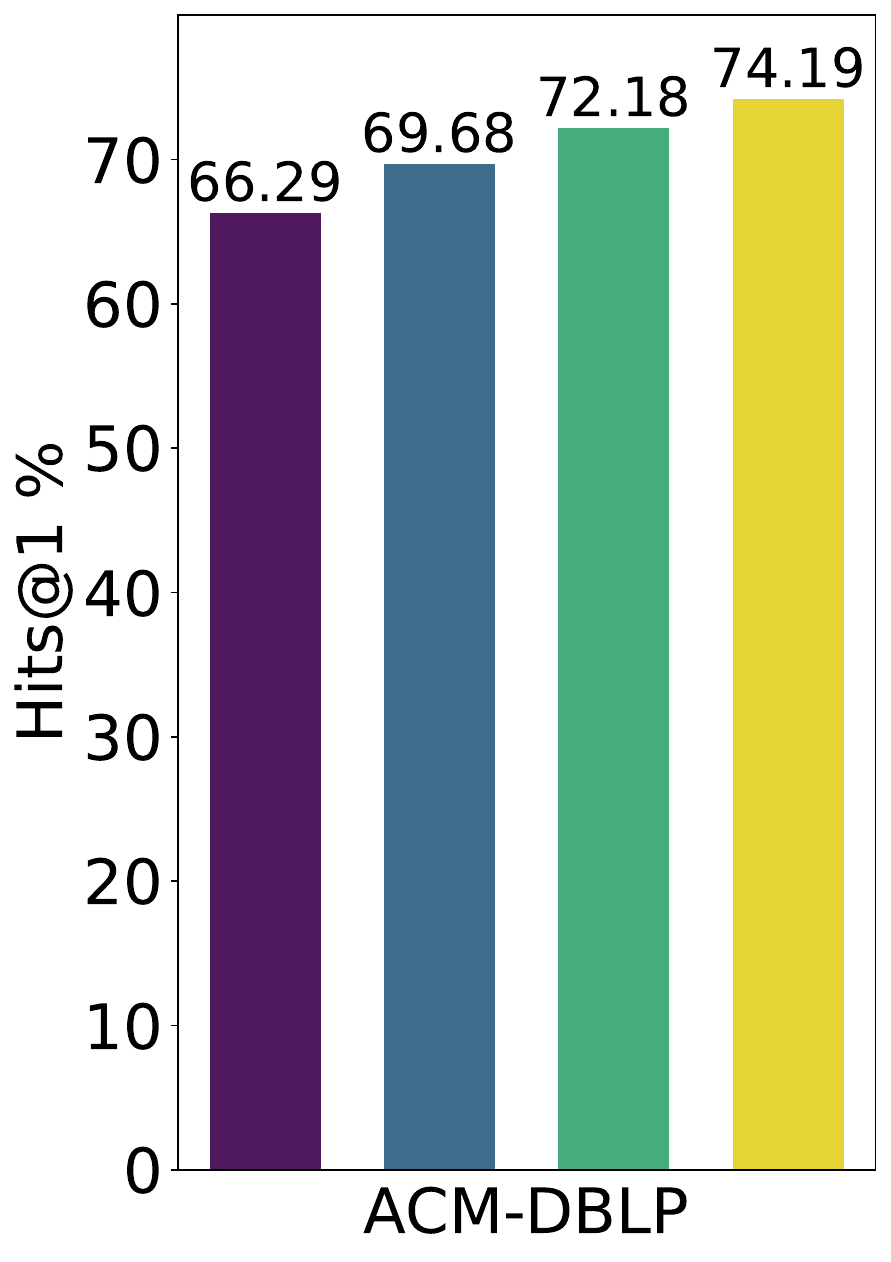} & 
        \hspace{-4mm}\includegraphics[height=40mm]{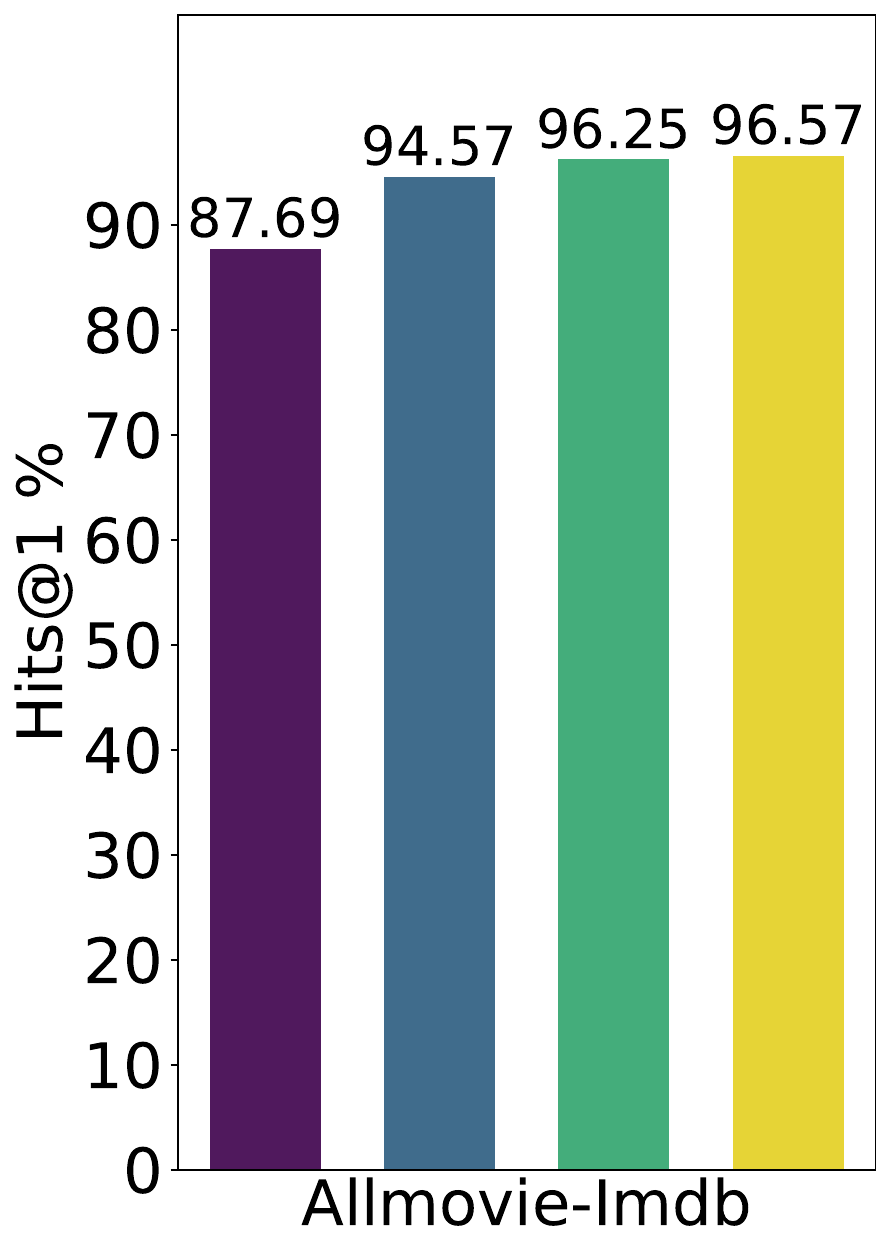}\\
    \end{tabular}
    \caption{
        Ablation study on three real-world datasets.
    }
    \label{fig:ablation}
\end{figure}

\subsubsection{Ablation Study}
To test the effectiveness of each module, we progressively remove \texttt{Combine}, \texttt{NUM}, and feature transformation in \texttt{GRAFT}. 
The resulting model variants are as follows.
For \texttt{CombAlign w/o C}, we make predictions by $\mathbf{T}_{GW}$ because it generally achieves higher accuracy by combining the prior knowledge from \texttt{WL}. 
\texttt{CombAlign w/o C+NUM} further eliminates the \texttt{NUM} and \texttt{WL} modules, therefore, the uniform marginal is adopted. 
The \texttt{CombAlign w/o C+NUM+FT} variant removes feature transformation as well as feature propagation, resulting in a model theoretically less expressive than   \texttt{SLOTAlign}. 

We show the results of the ablation study in Figure~\ref{fig:ablation}. 
It is clear that all proposed modules have non-negligible contributions to model accuracy. 
More specifically, the \texttt{Combine} module boosts the accuracy on Douban (resp. ACM-DBLP) by more than two percentage points. 
The non-uniform marginal and feature transformation significantly improve the model performance.

    \begin{figure*}[!htbp]
\centering
\begin{tabular}{ccccc} 
    \centering
      \hspace{-4mm}\includegraphics[ height=30mm]{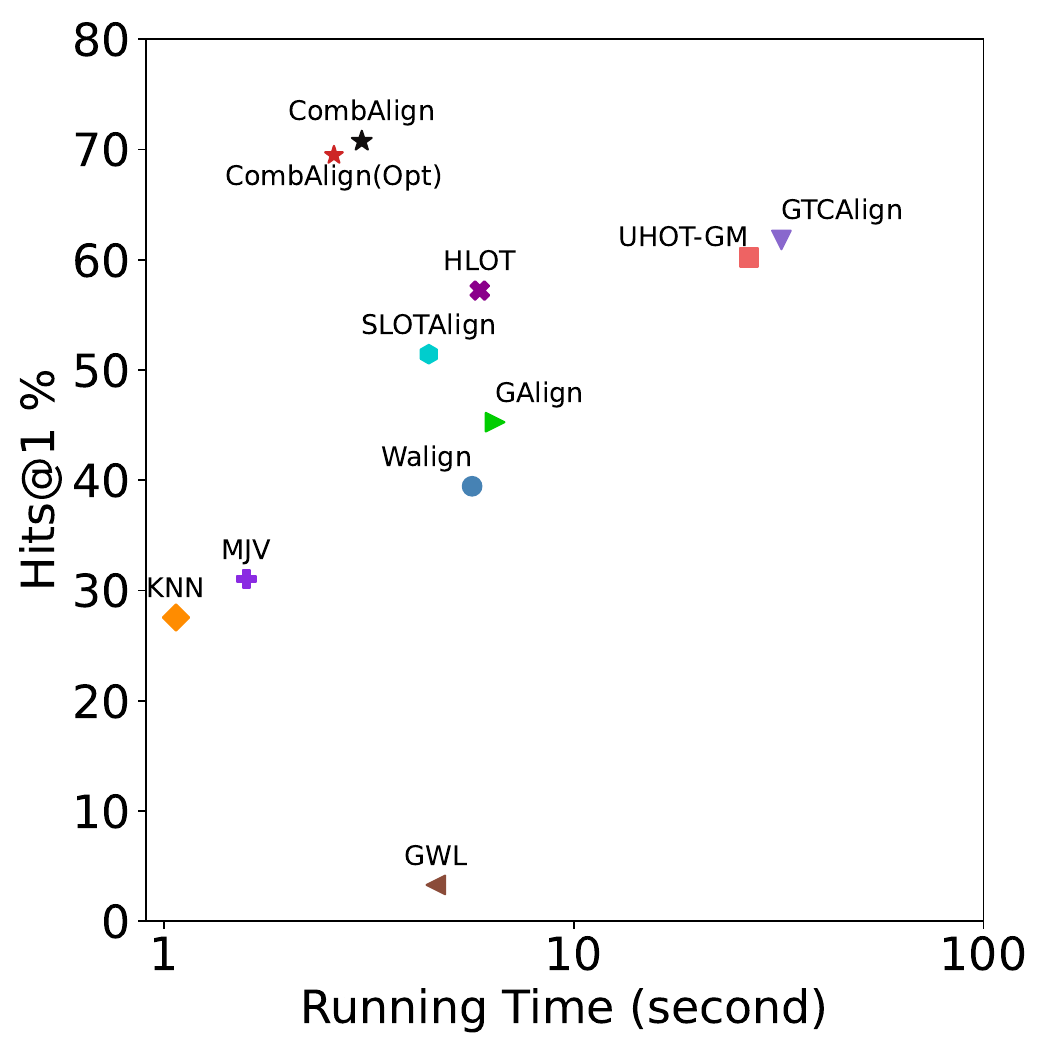}&
        \hspace{-4mm}\includegraphics[ height=29.5mm]{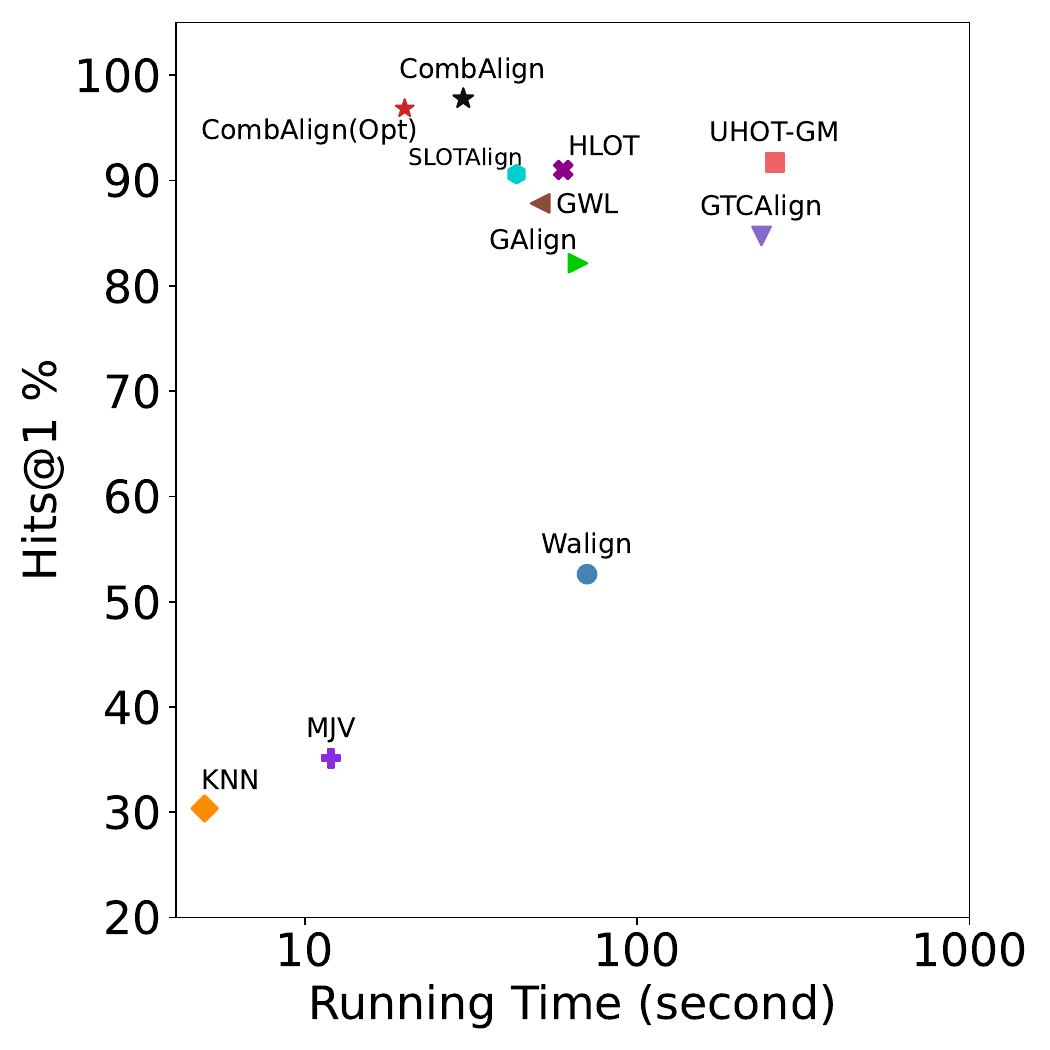}&
          \hspace{-4mm}\includegraphics[ height=30mm]{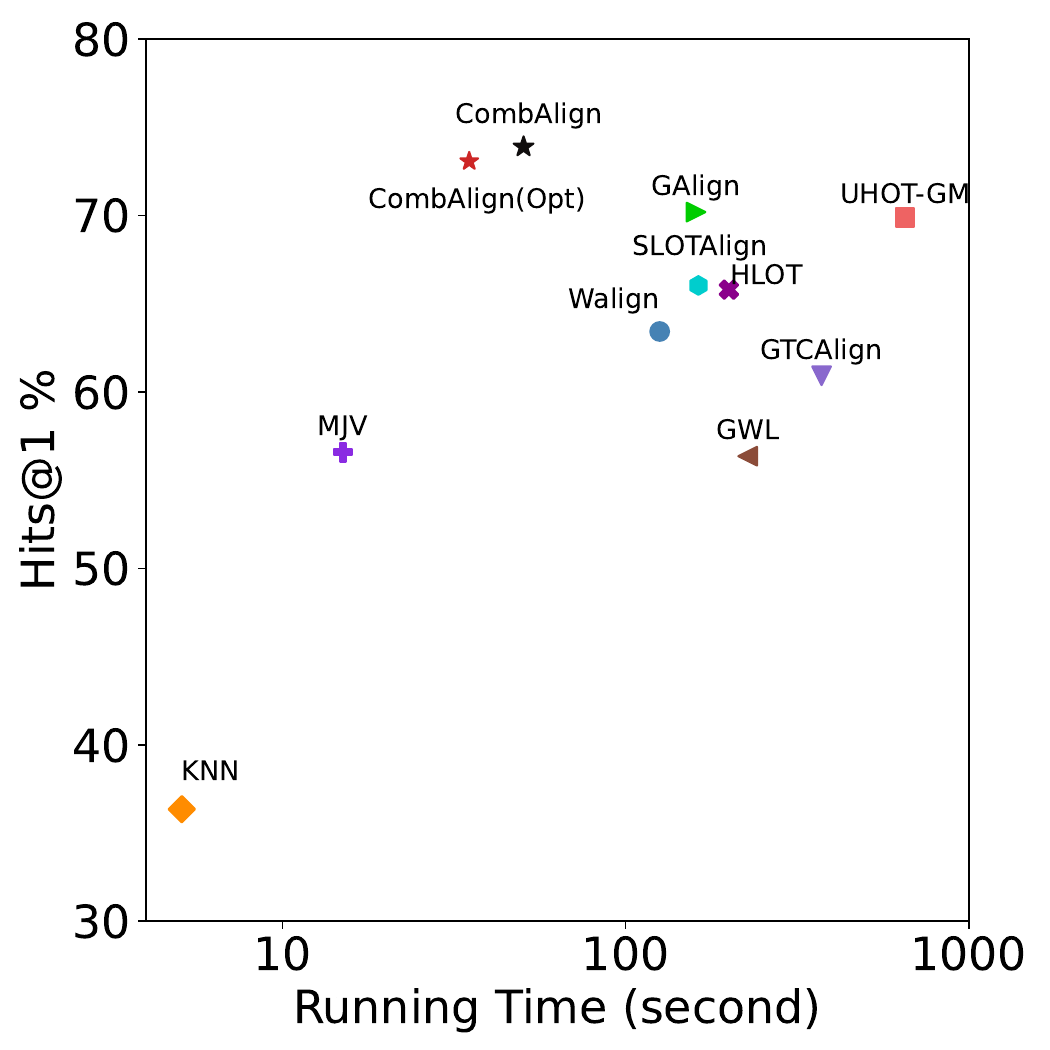}&
    \hspace{-4mm}\includegraphics[ height=30mm]{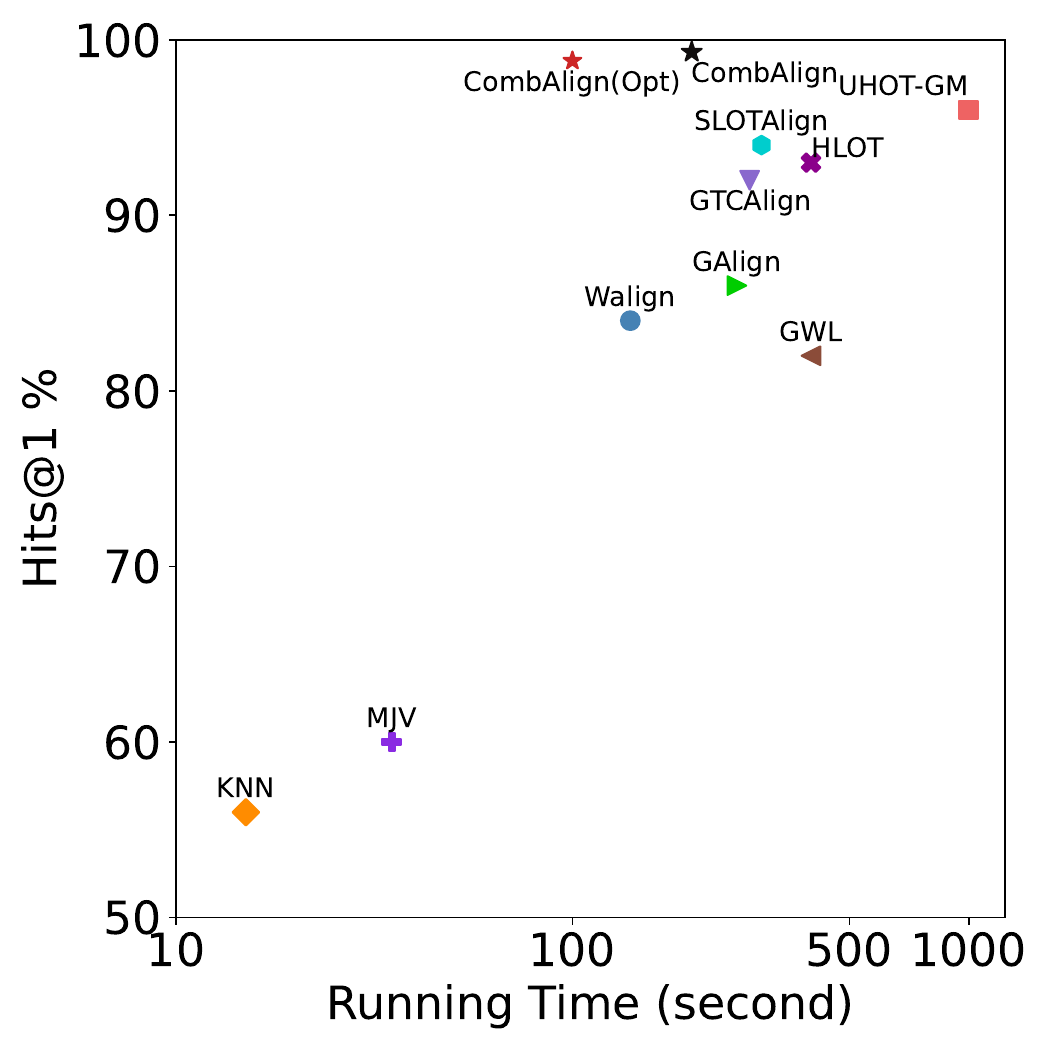}&
    \hspace{-4mm}\includegraphics[ height=30mm]{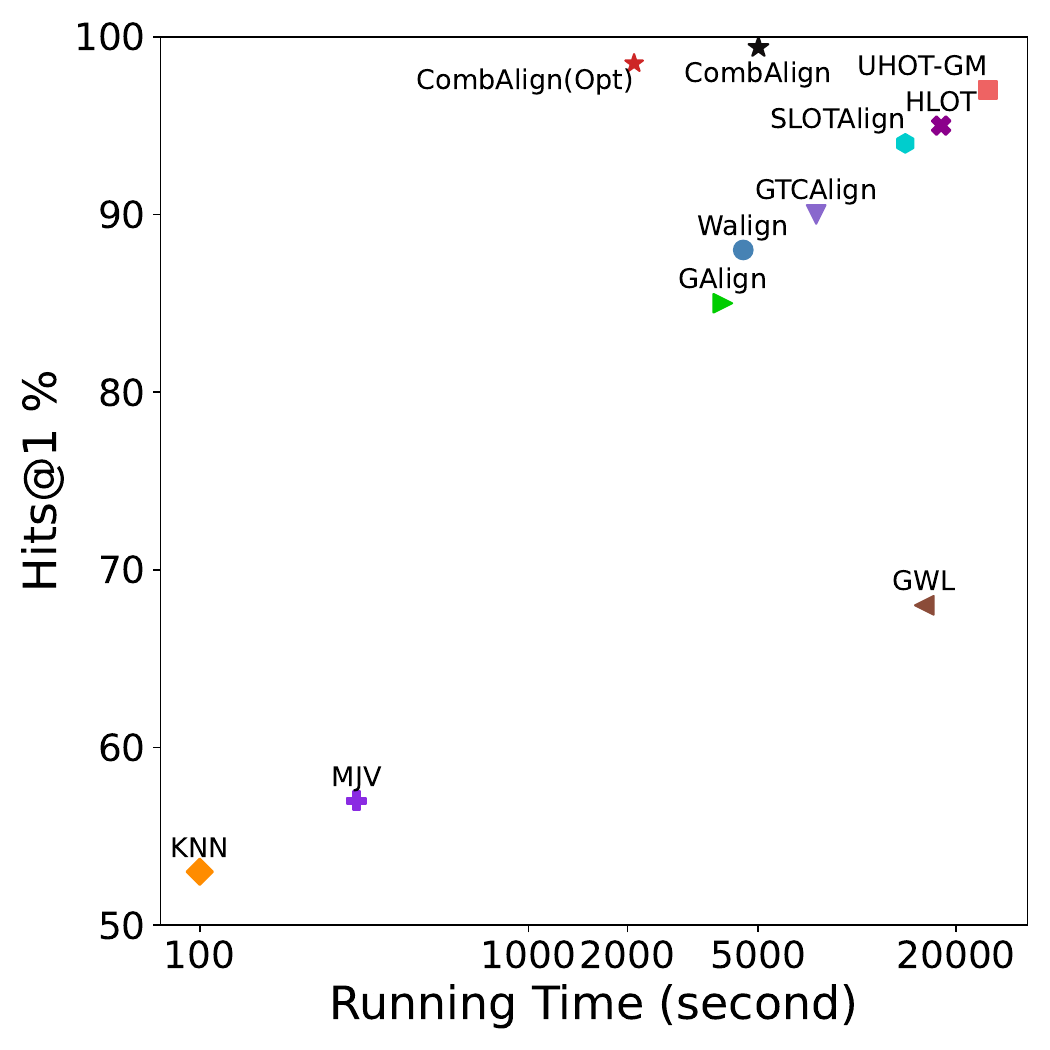}
    \\ 
  
    {(a) Douban} & {(b) Allmovie-Imdb} & {(c) ACM-DBLP} & {(d) CS} & {(e) Physics}  \\
\end{tabular}
\caption{Running time vs. Hits@1 of different methods.}
\label{fig:runtime}
\end{figure*}

\subsubsection{Further Evaluation of Different Modules} 
We conduct a more extensive investigation of the proposed modules. 

\noindent \texttt{GRAFT} \textbf{with different GNNs}. 
According to the theoretical results of Corollary~\ref{col:GNN}, more expressive GNNs should lead to better model accuracy. 
We employ the \texttt{CombAlign w/o C+NUM} variant to test the lightweight GCN, GCN, and GIN (Cf. Section~\ref{sec:model-GRAFT}). 
We also include an efficient Graph Transformer~\cite{wu2023sgformer} with even more expressive power. 
The results nicely matches our theoretical conclusion (see Table~\ref{tab: diff-gnn}). 

Note that by default the \texttt{GRAFT} module uses GCN (e.g., for Table~\ref{table:result}). 
We also point out that the GIN variant in Table~\ref{tab: diff-gnn} has already outperformed all baselines shown in Table~\ref{table:result}. 
We believe there is still room for further improvement with more powerful GNNs (and Graph Transformers).

\noindent \textbf{Baselines with the} \texttt{Combine} \textbf{module}. 
Next, we validate our \texttt{Combine} module for improving existing OT-based solutions. 
We substitute the \texttt{GRAFT} module with \texttt{GWL}, \texttt{SLOTAlign} and \texttt{UHOT-GM}, and report the final prediction with \texttt{Combine}. 
As shown in Table~\ref{tab:EL-other ot-based},
all methods achieve a significant enhancement of Hits@1, further validating the effectiveness of our \texttt{Combine} module. 

\begin{table}
\centering
\caption{Improving other OT-based models w/ \texttt{Combine} (Hits@1).}
\begin{tabular}{|p{22mm}|ccc|}
\hline
Methods & Douban & ACM-DBLP &Allmovie-Imdb \\
\hline

\texttt{GWL} & 3.29  & 56.36 & 87.82\\
\hline
\texttt{GWL+C} & 5.81  & 63.46 & 90.07  \\
\hline
\texttt{SLOTAlign} & 51.43  & 66.04  & 90.60 \\
\hline
\texttt{SLOTAlign+C} & 60.01  & 69.20  & 91.03  \\
\hline
\texttt{UHOT-GM} & 60.23  &69.89  & 91.73\\
\hline
\texttt{UHOT-GM+C} & 63.17  &71.32   & 92.45  \\
\hline
\end{tabular}
\label{tab:EL-other ot-based}
\end{table}

\begin{table}
\centering
\caption{The ratio of the remaining set over ground truth after eliminating one-to-many predictions.}
\begin{tabular}{|p{26mm}|ccc|}
\hline
\texttt{CombAlign} & Douban & ACM-DBLP &Allmovie-Imdb \\
\hline
 - w/ \texttt{C} & \textbf{100\%} & \textbf{100\%}  &\textbf{100\%}  \\
\hline
 - w/o \texttt{C} & 94.36\% &  88.36\% & 98.54\% \\
\hline
 - w/o \texttt{C+NUM} & 94.18\% & 86.50\% & 97.60\% \\
\hline
 - w/o \texttt{C+NUM+FT} & 91.32\% & 83.27\% & 96.25\% \\
\hline
\end{tabular}
\label{tab:ratio_combAlign}
\end{table}

\noindent \textbf{One-to-one matching property of} \texttt{CombAlign}. 
In Table~\ref{tab:ratio_combAlign}, we show the ratio of the remaining set compared to the ground truth set after eliminating one-to-many predictions for all model variants of \texttt{CombAlign}. 
This step is similar to that of Figure~\ref{fig:missingrate} (see Section~\ref{sec:model-EL}), in which we demonstrate the ratio of one-to-many predictions for our baselines. 
Consistent with our analysis, \texttt{CombAlign} manages to return a set of node pairs with one-to-one matching.
Other variants of \texttt{CombAlign} also ensures higher similarity between matched pairs compared to the baselines. 

\begin{figure}[!htbp]
\centering
\begin{tabular}{cccc} 
    \centering
    \hspace{-4mm}\includegraphics[ height=20mm]{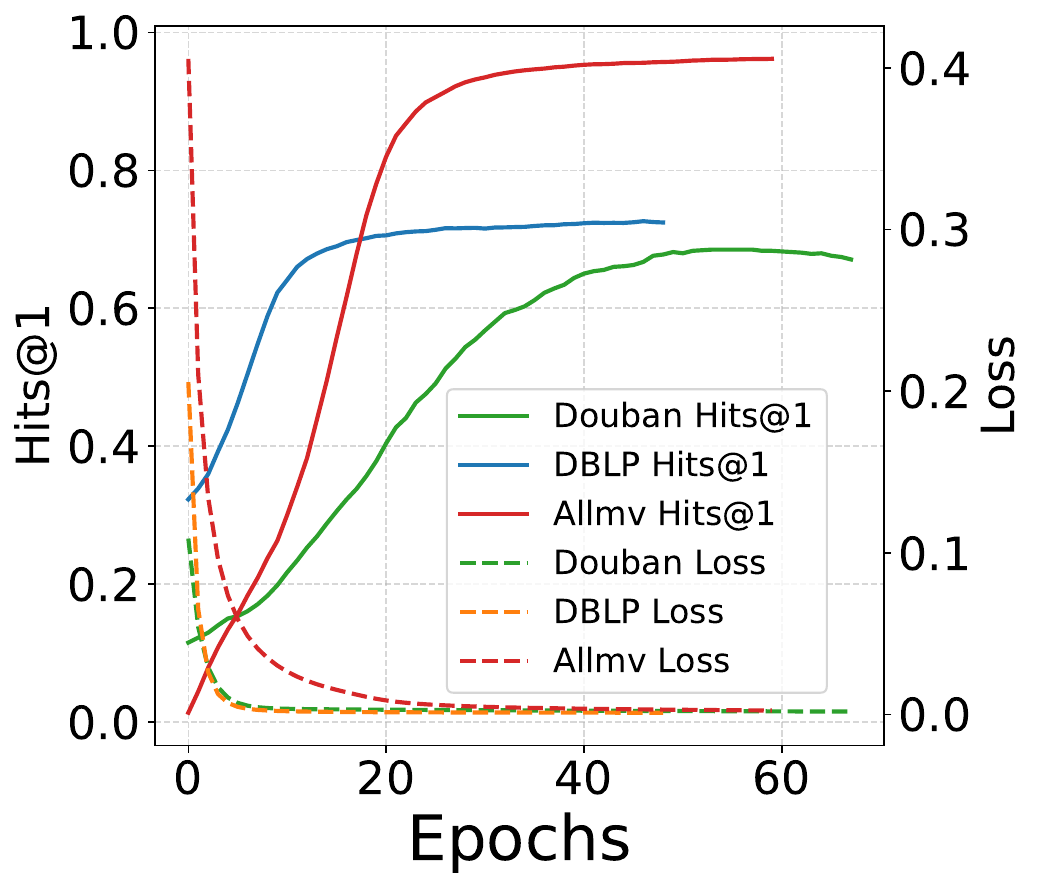}&
    \hspace{-4mm}\includegraphics[ height=20mm]{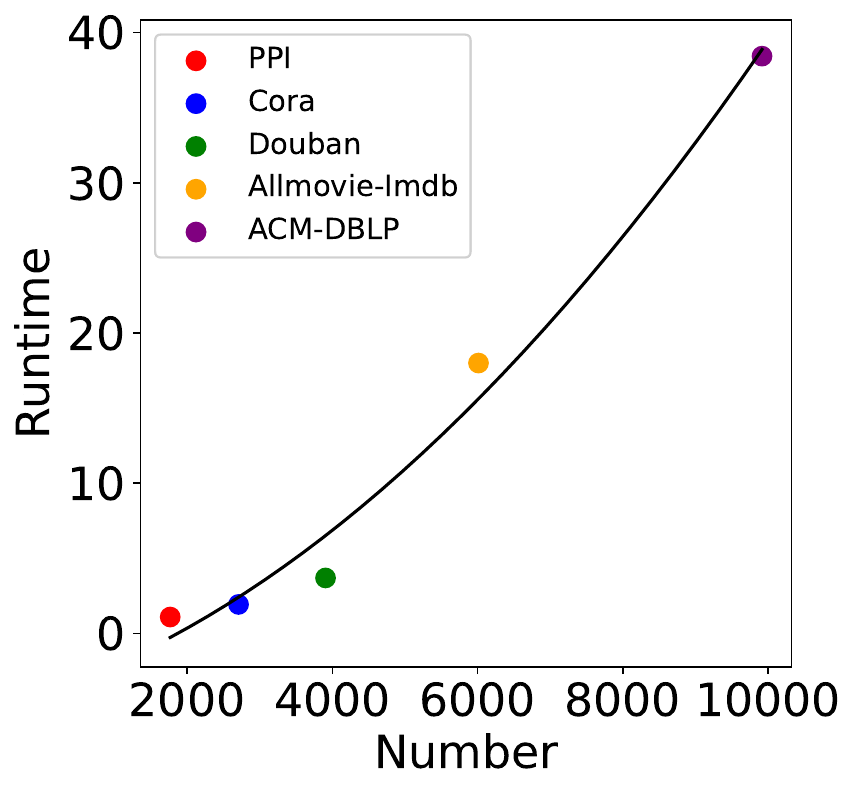}
    &
    \hspace{-4mm}\includegraphics[ height=20mm]{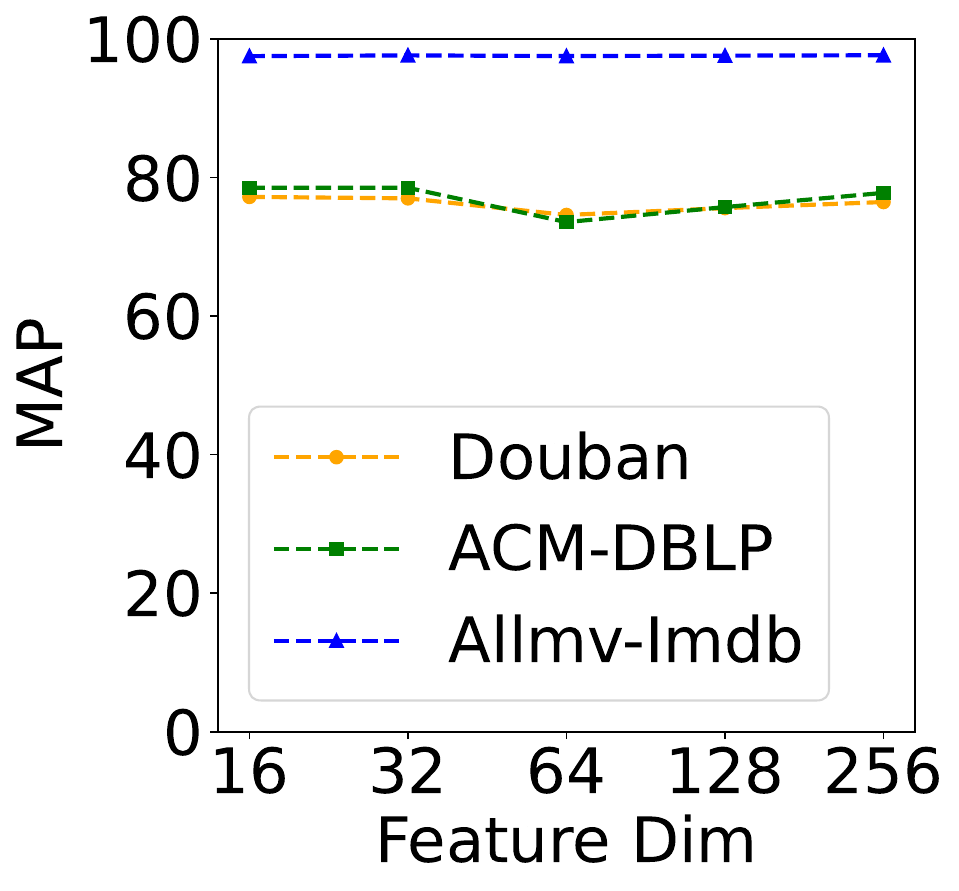}
    &
    \hspace{-4mm}\includegraphics[ height=20mm]{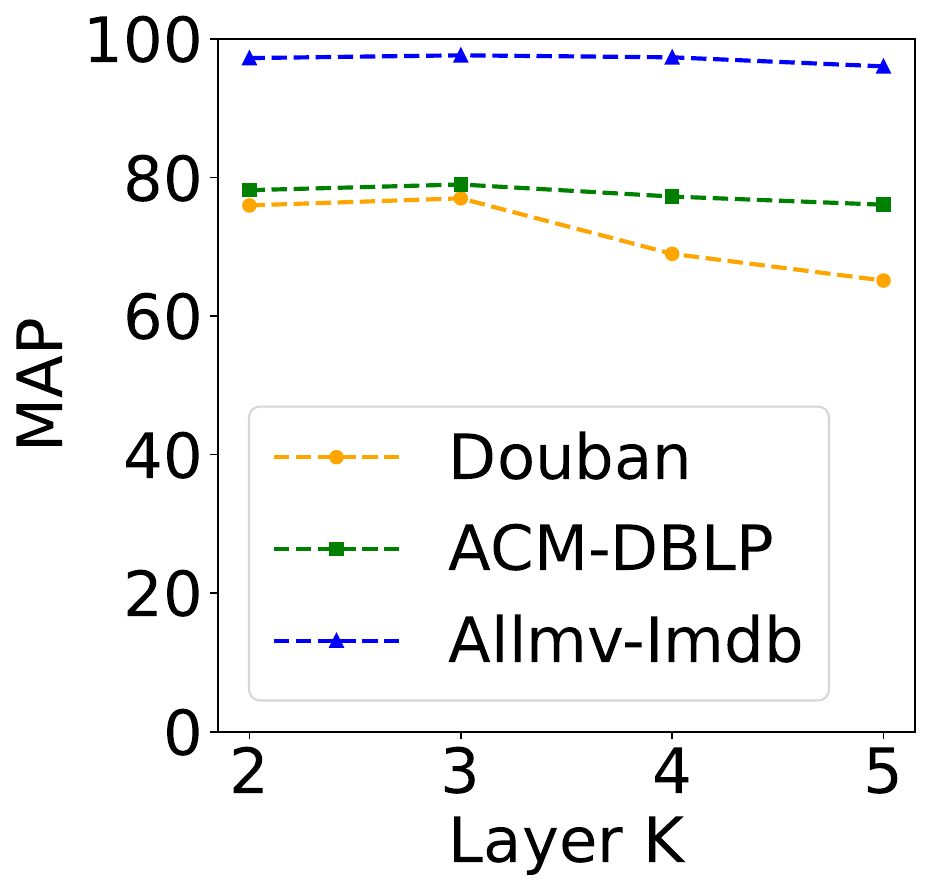}
    \\ 
  
    {(a)} & {(b) }  & {(c) } & {(d) } \\
\end{tabular}
\caption{ Further Evaluation. (a) Convergence. (b) Runtime of \texttt{CombAlign (Opt)} vs. graph size. (c) and (d) represent the sensitivity analysis.} 
\label{fig:crh}
\end{figure}

\subsubsection{Convergence Analysis} \label{sec:efficiency-comparison}
We simultaneously visualize the prediction accuracy and the model loss (i.e., Equation~\ref{eqn:GWD}) on three real-world datasets at different epochs, as shown in Figure~\ref{fig:crh}(a). 
With the increase of epochs, the loss gradually decreases to almost approaching zero, meanwhile, the Hits@1 value gradually converges.  

\subsubsection{Model Efficiency and Scalability} 
For efficiency analysis, we compare the running time of all methods in Figure~\ref{fig:runtime} on five datasets. 
It can be observed that \texttt{KNN} and \texttt{MJV} achieve the highest efficiency but suffer from poor performance. 
Our optimized algorithm, \texttt{CombAlign (Opt)}, has the best efficiency among all learning-based methods with better accuracy. Moreover, compared to the best-performing baselines, the speedup is about 10$\times$.  
Even the \texttt{CombAlign} algorithm outperforms most baselines in terms of efficiency. 
This can be attributed to the strong expressive power of our model, which facilitates faster convergence. 
Note that the GW learning process takes the majority of the time, while the additional costs of \texttt{WL} and \texttt{Combine} modules are negligible. 
The fitted runtime curve versus graph size (Figure~\ref{fig:crh}(b)) aligns with the algorithm’s quadratic time complexity.

To test the model's scalability on large-scale graphs, we integrate our solution with a divide-and-conquer framework, \texttt{S3GA-RCF}~\cite{guo2025s3ga}. 
We first partition the source and target graphs into smaller subgraphs and apply our model to the subgraphs, and then aggregate the matching results.
As shown in Table~\ref{tab:large-r4}, EN-FR and EN-DE represent two datasets with one million nodes~\cite{guo2025s3ga}, where our model achieves higher accuracy with comparable time cost, validating its effectiveness. 
It is worth noting that the quality of the partitioning significantly affects the overall accuracy. 

\begin{table}
\centering
\small
\caption{ Comparison of Hits@1 on the datasets with 1 million nodes.}
\vspace{-1mm}
\begin{tabular}{|c|c|c|}
\hline
 dataset&  \texttt{S3GA-RCF} &\texttt{CombAlign} \\
 \hline

EN-FR  & 71.77 (879 secs) & 72.98 (903 secs)  \\
\hline
EN-DE   & 71.93 (412 secs) & 73.51 (427 secs)   \\
\hline
\end{tabular}
\label{tab:large-r4}
\end{table}

\subsubsection{Sensitivity Analysis} 

To examine the impact of different hyper-parameters, on three real datasets including Douban, ACM-DBLP and Allmovie-Imdb, we analyze the performance of \texttt{CombAlign} with different feature dims $d \in \left \{16,32,64,128,256  \right \} $ and the number of graph convolution layers $K \in \left\{ 2,3,4,5\right\}$, as illustrated in Figure~\ref{fig:crh}(c) and (d).
In general, the model performance is stable w.r.t. $d$. Meanwhile, the accuracy declines when $K$ exceeds 3, consistent with the oversmoothing phenomenon of GNNs. 

\subsubsection{Robustness Analysis} To validate the model’s robustness, we follow prior work~\cite{slotalign,dhot} and perturb varying proportions of edges on three synthetic datasets. 
As shown in Figure~\ref{fig:noise_accuracy}, in which we select the representative method among traditional, embedding-based and OT-based solutions, when we increase the proportion of perturbed edges, the performance of all models gradually decline. 
When the perturbation ratio is relatively low (e.g., below 10\%), all methods, including \texttt{KNN}, achieve relatively high accuracy. 
Our model, \texttt{CombAlign}, demonstrates stronger robustness even with a high noise level. 
In this case, learning-based approaches demonstratre clear advantages over traditional algorithms such as \texttt{KNN} even on these synthetic datasets. 

\section{Conclusion and Future Work}\label{conclusion}

In this paper, we formally investigate the model expressiveness for unsupervised graph alignment from two theoretical perspectives. 
We characterize the model's discriminative power as correctly distinguishing matched and unmatched node pairs across graphs and study the model's capability of guaranteeing important matching properties including one-to-one node matching and mutual alignment. 
Then, we propose a hybrid model named \texttt{CombAlign} to combine the advantages of both optimal transport and embedding-based solutions, which leads to provably more discriminative power. 
Our method is further empowered by ensembling both OT and embedding-based predictions with a traditional algorithm-inspired strategy, i.e., maximum weight matching, to guarantee the matching properties. 
Extensive experiments demonstrate significant improvements in alignment accuracy and the effectiveness of the proposed modules. 
Moreover, our theoretical analysis of expressive power is confirmed. 
For future work, we aim to enhance the model's efficiency on large-scale graphs.

\section*{Acknowledgments}
The authors would like to thank Prof. Lei Chen, Prof. Xiaofang Zhou, and Prof. Zhi Jin for fruitful discussions, and Haoran Cheng, Prof. Dixin Luo, and Prof. Hongteng Xu for their help in providing the baseline codes. 
This work is supported by NSFC under grants 62202037, 62372031, and 62372034. 

\section*{Appendix}
\subsection{Additional Theoretical Results}
The following corollary states that by integrating multiple views into the transport cost design (e.g.,~\cite{slotalign, dhot}), feature transformation still provides strictly better expressiveness. 
\begin{corollary} \label{col:A1}
Assume that the intra-graph cost matrix $\mathbf{C}_p$ is the linear combination of graph structure $\mathbf{A}_p$, feature information $\mathbf{X}_p$, and node embedding $\mathbf{Z}_p$, i.e.,
\begin{equation} \label{eqn:multi-view}
\mathbf{C}_p = \beta_p^{(1)} \mathbf{A}_p + \beta_p^{(2)} \mathbf{X}_p\mathbf{X}_p^\intercal + \beta_p^{(3)} \mathbf{Z}_p\mathbf{Z}_p^\intercal, p=s,t,
\end{equation}
where we use $\boldsymbol{\beta}_p$ to represent learnable coefficients. If feature transformation is applied to compute $\mathbf{Z}_p$, the costs have more discriminative power in separating the matched and unmatched node pairs. 
\end{corollary}
\begin{proof}
Consider the example in the proof of Theorem~\ref{thm:T1}. It can be derived that the first two terms of Equation~\ref{eqn:multi-view} are equal to zero. Following Theorem~\ref{thm:T1}, feature transformation still provides more discriminative power. 
\end{proof}

With mild assumptions, e.g., by adopting the lightweight GCN~\cite{wu2019simplifying} for the \texttt{GRAFT} module, the convergence of \texttt{CombAlign} is theoretically guaranteed as in~\cite{gwl, slotalign}.
\begin{theorem} \label{thm:converge}
    Denote by $L(\mathbf{T}, \boldsymbol{\beta}, \mathcal{W})$ the learning objective in Equation~\ref{eqn:GWD-2}. 
    Suppose that $0 < \tau_T < \frac{1}{L^{T}_f}, 0 < \tau_\beta < \frac{1}{L^{\beta}_f}$, and $0 < \tau_W < \frac{1}{L^{W}_f}$, where $L^{T}_f$, $L^{\beta}_f$, and $L^{W}_f$ are the gradient Lipschitz continuous modulus of $L(\mathbf{T}, \boldsymbol{\beta}, \mathcal{W})$ respectively. 
    If the following conditions hold, i.e.,
    \begin{align}
    \mathcal{T} &= \{ \mathbf{T} \geq 0: \mathbf{T} \boldsymbol{1}_{n_2}= \boldsymbol{\mu} ,\mathbf{T}^{\intercal}\boldsymbol{1}_{n_1} =\boldsymbol{\nu} \}, \label{eqn:thm2-1} \\
    \mathcal{B} &= \{ (\boldsymbol{\beta}_s, \boldsymbol{\beta}_t) \geq 0: \sum_{i=1}^3 \beta^{(i)}_p = 1, p=s,t \}, \label{eqn:thm2-2}\\
    \mathcal{W} &= \{  \sum_{i=1} \mathbf{W}_{ij} = 1, \forall j \in [1,d] \}, \label{eqn:thm2-3}
    \end{align}
    then the GW learning process converges to a critical point of $\bar{L}(\mathbf{T}, \boldsymbol{\beta}, \mathcal{W})$, with
    $
        \bar{L}(\mathbf{T}, \boldsymbol{\beta}, \mathcal{W}) = L(\mathbf{T}, \boldsymbol{\beta}, \mathcal{W}) + \mathbb{I}_T(\mathbf{T}) + \mathbb{I}_{B}(\boldsymbol{\beta}) + \mathbb{I}_W(\mathcal{W}),
    $
    where $\mathbb{I}_C(\cdot)$ denotes the indicator function on the set $C$. 
\end{theorem}
\begin{proof} 
Note that by our definition, $\mathcal{T}, \mathcal{B}$, and $\mathcal{W}$ are bounded sets, and $L(\mathbf{T}, \boldsymbol{\beta}, \mathcal{W})$ is a bi-quadratic function with respect to $\mathbf{T}$, $\boldsymbol{\beta}$, and $\mathcal{W}$. 
To guarantee that $\mathcal{W}$ satisfies the above constraint in each iteration, it is sufficient to apply an activation function (e.g., ReLU) followed by column-wise normalization. 
The proof then generally follows that of Theorem 5 in~\cite{slotalign}, where we can compute the gradient of $L$ w.r.t. $\mathbf{T}$, $\boldsymbol{\beta}$, and $\mathcal{W}$ to update the parameters with provable convergence.
\end{proof}

\subsection{Proofs of Theorems and Corollaries in Section~\ref{sec:model}}

\begin{proof} [Proof of Theorem~\ref{thm:T1}] 
First note that if an algorithm is able to separate $ |\mathbf{C}_s(i, j) - \mathbf{C}_t(k, l)|$ and $ |\mathbf{C}_s(i, j') - \mathbf{C}_t(k, l)|$, or more specifically, to learn that $|\mathbf{C}_s(i, j) - \mathbf{C}_t(k, l)| < |\mathbf{C}_s(i, j') - \mathbf{C}_t(k, l)|$, it is more likely to make the correct prediction ($\mathbf{T}(j, l) > \mathbf{T}(j', l)$). 
Now assume that $\mathbf{X}_s$ contains different one-hot features for each node, and the distance between $u_i$ and $u_j$ (resp. $u_{j’}$) is beyond twice of $K$, the maximum step of feature propagation. For simplicity, let $i$, $j$, and $j’$ be the corresponding index. 
Since $u_i$ cannot propagate any information to $u_j$ and $u_{j’}$ (and vice versa), we have $\mathbf{R}_s(u_i, j) = 0, \mathbf{R}_s(u_i, j’) = 0, \mathbf{R}_s(u_j, i) = 0$, and $\mathbf{R}_s(u_{j’}, i) = 0$. 
Thus, we have $\mathbf{R}_s(u_i)^\intercal \mathbf{R}_s(u_j) = \mathbf{R}_s(u_i)^\intercal \mathbf{R}_s(u_{j’}) = 0$, i.e., $ \mathbf{C}'_s(i, j) = \mathbf{C}'_s(i, j') = 0$. 

With the linear transformation matrix $\mathbf{W}$, different dimensions of the one-hot features are directly interacted. We might have $\mathbf{Z}_s(u_i)^\intercal \mathbf{Z}_s(u_j) \neq 0$ and $\mathbf{Z}_s(u_i)^\intercal \mathbf{Z}_s(u_{j’}) = 0$ by setting the corresponding columns of $\mathbf{W}$. 
More generally, by learning $\mathbf{W}$, it is possible to have $\mathbf{Z}_s(u_i)^\intercal \mathbf{Z}_s(u_j) \neq \mathbf{Z}_s(u_i)^\intercal \mathbf{Z}_s(u_{j’})$, and the theorem follows.
\end{proof}


\begin{proof} [Proof of Corollary~\ref{col:GNN}]
 Consider the source graph $\mathcal{G}_s$, for $\mathcal{G}_t$, the analysis is similar.
    Let $\mathbf{Z}_s$ (resp. $\mathbf{Z}'_s$) be the node embeddings obtained by $f_{GNN}$ (resp. $f'_{GNN}$). 
    Since their differences in expressive power, suppose there exists two nodes $u_j, u_{j'} \in \mathcal{V}_s$ with $\mathbf{Z}_s(u_j) \neq \mathbf{Z}_s(u_{j'})$ and $\mathbf{Z}'_s(u_j) = \mathbf{Z}'_s(u_{j'})$. 
    Consider another node $u_i \in \mathcal{V}_s$. It is possible to have $\mathbf{Z}_s(u_i)^\intercal \mathbf{Z}_s(u_j) \neq \mathbf{Z}_s(u_i)^\intercal \mathbf{Z}_s(u_{j'})$ (via learning a proper $f_{GNN}$), but it must hold that $\mathbf{Z}_s(u_i)^\intercal \mathbf{Z}'_s(u_j) = \mathbf{Z}_s(u_i)^\intercal \mathbf{Z}'_s(u_{j'})$. 
    Therefore, we have $\mathbf{C}_s(i, j) \neq \mathbf{C}_s(i, j')$ and $\mathbf{C}'_s(i, j) = \mathbf{C}'_s(i, j')$. 
    Then the corollary holds following Theorem~\ref{thm:T1}.

\end{proof}


\begin{proof} [Proof of Theorem~\ref{thm:NUM}]
 First, we show that there exists such a case satisfying our assumption. 
 If the embeddings in both $\mathbf{X}_p$ and $\mathbf{Z}_p$ are normalized, we have $\mathbf{X}_p(v)^\intercal \mathbf{X}_p(v) = \mathbf{Z}_p(v)^\intercal \mathbf{Z}_p(v) = 1$, thus the first constraint holds. 
 The second constraint is satisfied for undirected graphs if each term of $\mathbf{C}_p$ is computed by the dot product of node embeddings. 
 For the third constraint, an illustrative example is given in Figure~\ref{fig:NUM}(b).
 For two nodes $v_k$ and $v_{k’}$ in $\mathcal{G}_t$ that are automorphism to each other but with different one-hot encodings, after feature propagation, we have $\mathbf{R}_t(v_l, k) = \mathbf{R}_t(v_l, k’)$ for each $v_l$ because the propagation strategy is shared by each dimension.
 Therefore, we have $\mathbf{R}_t(v_l)^\intercal \mathbf{R}_t(v_k) = \mathbf{R}_t(v_l)^\intercal \mathbf{R}_t(v_{k’})$, and the constraint follows. 
 This claim also holds with feature transformation if the GNN module cannot separate $\mathbf{Z}_t(v_l)^\intercal \mathbf{Z}_t(v_k)$ and $\mathbf{Z}_t(v_l)^\intercal \mathbf{Z}_t(v_{k’})$.

 Second, we categorize the terms in $\mathbf{C}_{gwd}(i, k)$ as the following cases (Cf. Equation~\ref{eqn:GWD-ik}). Note that we have $\mathbf{T}^{(0)}(i, k) = \frac{1}{n_1n_2}, \forall i,k$. Denote by $c_{ik}(j, l)$ the terms in $\mathbf{C}_{gwd}(i, k)$, we have
 \begin{itemize} 
 \item $l=k$: $c_{ik}(j, k) = \left | \mathbf{C}_s(i, j) - \mathbf{C}_t(k, k) \right |^2 \mathbf{T}(j, k)$, 

 \item $l=k’$: $c_{ik}(j, k') = \left | \mathbf{C}_s(i, j) - \mathbf{C}_t(k, k’) \right |^2 \mathbf{T}(j, k’)$,

 \item $l \neq k, k’$: $c_{ik}(j, l) = \left | \mathbf{C}_s(i, j) - \mathbf{C}_t(k, l) \right |^2 \mathbf{T}(j, l)$.
 \end{itemize}
 For each $j$, by constraint 1 and 2, the summation of the first and second cases are equal for $\mathbf{C}_{gwd}(i, k)$ and $\mathbf{C}_{gwd}(i, k’)$. 
 By constraint 3, for each term in the third case, the values are the same for $v_k$ and $v_{k’}$. 
 Then we have $\mathbf{C}_{gwd}(i, k) = \mathbf{C}_{gwd}(i, k')$ in the first iteration of GW learning for each $u_i \in \mathcal{G}_s$, which completes the proof. 
 \end{proof}

\begin{proof} [Proof of Theorem~\ref{thm:combine}]

Note that the invoked maximum weight matching algorithm finds a set of non-adjacent edges, each of which connects a pair of nodes across two graphs. 
Hence, each node is adjacent to at most one chosen edge, satisfying the constraint of one-to-one matching. 
Similarly, the non-adjacent property of the matching edges naturally guarantees mutual alignment.
\end{proof}

\bibliographystyle{IEEEtran}
\bibliography{reference}

@String{JACM = "J. ACM" }

@String{Computing = "Computing" }

@String{Computer = "{IEEE} Computer" }

@String{Academic = "Academic Press" }

@String{Springer = "Springer-Verlag" }

@ArtifactSoftware{R,
    title = {R: A Language and Environment for Statistical Computing},
    author = {{R Core Team}},
    organization = {R Foundation for Statistical Computing},
    address = {Vienna, Austria},
    year = {2019},
    url = {https://www.R-project.org/},
}

@inproceedings{slotalign,
  title={Robust attributed graph alignment via joint structure learning and optimal transport},
  author={Tang, Jianheng and others},
  booktitle={2023 IEEE 39th International Conference on Data Engineering (ICDE)},
  pages={1638--1651},
  year={2023},
  organization={IEEE}
}

@article{zitnik2017predicting,
  title={Predicting multicellular function through multi-layer tissue networks},
  author={Zitnik, Marinka and others},
  journal={Bioinformatics},
  volume={33},
  number={14},
  pages={i190--i198},
  year={2017},
  publisher={Oxford University Press}
}

@article{sen2008collective,
  title={Collective classification in network data},
  author={Sen, Prithviraj and others},
  journal={AI magazine},
  volume={29},
  number={3},
  pages={93--93},
  year={2008}
}

@inproceedings{galign,
  title={Adaptive network alignment with unsupervised and multi-order convolutional networks},
  author={Trung, Huynh Thanh and others},
  booktitle={2020 IEEE 36th International Conference on Data Engineering (ICDE)},
  pages={85--96},
  year={2020},
  organization={IEEE}
}

@inproceedings{zhang2016final,
  title={Final: Fast attributed network alignment},
  author={Zhang, Si and Tong, Hanghang},
  booktitle={Proceedings of the 22nd ACM SIGKDD international conference on knowledge discovery and data mining},
  pages={1345--1354},
  year={2016}
}

@article{zhang2018attributed,
  title={Attributed network alignment: Problem definitions and fast solutions},
  author={Zhang, Si and Tong, Hanghang},
  journal={IEEE Transactions on Knowledge and Data Engineering},
  volume={31},
  number={9},
  pages={1680--1692},
  year={2018},
  publisher={IEEE}
}

@inproceedings{gwl,
  title={Gromov-wasserstein learning for graph matching and node embedding},
  author={Xu, Hongteng and others},
  booktitle={International conference on machine learning},
  pages={6932--6941},
  year={2019},
  organization={PMLR}
}

@inproceedings{walign,
  title={Unsupervised graph alignment with wasserstein distance discriminator},
  author={Gao, Ji and others },
  booktitle={Proceedings of the 27th ACM SIGKDD Conference on Knowledge Discovery \& Data Mining},
  pages={426--435},
  year={2021}
}

@inproceedings{zeng2023parrot,
  title={PARROT: Position-Aware Regularized Optimal Transport for Network Alignment},
  author={Zeng, Zhichen and others},
  booktitle={Proceedings of the ACM Web Conference 2023},
  pages={372--382},
  year={2023}
}

@article{gtcalign,
  title={GTCAlign: Global Topology Consistency-based Graph Alignment},
  author={Wang, Chenxu and others},
  journal={IEEE Transactions on Knowledge and Data Engineering},
  year={2023},
  publisher={IEEE}
}

@article{gcn,
  title={Semi-supervised classification with graph convolutional networks},
  author={Kipf, Thomas N and Welling, Max},
  journal={arXiv preprint arXiv:1609.02907},
  year={2016}
}

@article{gat,
  title={Graph attention networks},
  author={Veli{\v{c}}kovi{\'c}, Petar and others},
  journal={arXiv preprint arXiv:1710.10903},
  year={2017}
}

@inproceedings{peyre2016gromov,
  title={Gromov-wasserstein averaging of kernel and distance matrices},
  author={Peyr{\'e}, Gabriel and others},
  booktitle={International conference on machine learning},
  pages={2664--2672},
  year={2016},
  organization={PMLR}
}

@article{km,
  title={Theoretical improvements in algorithmic efficiency for network flow problems},
  author={Edmonds, Jack and others},
  journal={Journal of the ACM (JACM)},
  volume={19},
  number={2},
  pages={248--264},
  year={1972},
  publisher={ACM New York, NY, USA}
}

@inproceedings{km2,
  title={A shortest augmenting path algorithm for dense and sparse linear assignment problems},
  author={Jonker, Roy and Volgenant, Ton},
  booktitle={DGOR/NSOR: Papers of the 16th Annual Meeting of DGOR in Cooperation with NSOR/Vortr{\"a}ge der 16. Jahrestagung der DGOR zusammen mit der NSOR},
  pages={622--622},
  year={1988},
  organization={Springer}
}

@article{munkres1957algorithms,
  title={Algorithms for the assignment and transportation problems},
  author={Munkres, James },
  journal={Journal of the society for industrial and applied mathematics},
  volume={5},
  number={1},
  pages={32--38},
  year={1957},
  publisher={SIAM}
}

@inproceedings{cui2016solving,
  title={Solving large-scale assignment problems by Kuhn-Munkres algorithm},
  author={Cui, Hong and others},
  booktitle={2nd international conference on advances in mechanical engineering and industrial informatics (AMEII 2016)},
  pages={822--827},
  year={2016}
}

@article{buvs2009towards,
  title={Towards auction algorithms for large dense assignment problems},
  author={Bu{\v{s}}, Libor and Tvrd{\'\i}k, Pavel},
  journal={Computational Optimization and Applications},
  volume={43},
  pages={411--436},
  year={2009},
  publisher={Springer}
}

@inproceedings{wang2017fora,
  title={FORA: simple and effective approximate single-source personalized pagerank},
  author={Wang, Sibo and others},
  booktitle={Proceedings of the 23rd ACM SIGKDD International Conference on Knowledge Discovery and Data Mining},
  pages={505--514},
  year={2017}
}

@article{schwartz1994computational,
  title={A computational analysis of the auction algorithm},
  author={Schwartz, BL},
  journal={European journal of operational research},
  volume={74},
  number={1},
  pages={161--169},
  year={1994},
  publisher={Elsevier}
}

@inproceedings{dietterich2000ensemble,
  title={Ensemble methods in machine learning},
  author={Dietterich, Thomas G},
  booktitle={International workshop on multiple classifier systems},
  pages={1--15},
  year={2000},
  organization={Springer}
}

@book{zhou2012ensemble,
  title={Ensemble methods: foundations and algorithms},
  author={Zhou, Zhi-Hua},
  year={2012},
  publisher={CRC press}
}

@inproceedings{yan2021bright,
  title={Bright: A bridging algorithm for network alignment},
  author={Yan, Yuchen and others},
  booktitle={Proceedings of the web conference 2021},
  pages={3907--3917},
  year={2021}
}

@inproceedings{liu2016aligning,
  title={Aligning Users across Social Networks Using Network Embedding.},
  author={Liu, Li and others },
  booktitle={Ijcai},
  volume={16},
  pages={1774--80},
  year={2016}
}

@article{dhot,
  title={DHOT-GM: Robust Graph Matching Using A Differentiable Hierarchical Optimal Transport Framework},
  author={Cheng, Haoran and others},
  journal={arXiv preprint arXiv:2310.12081},
  year={2023}
}

@article{shchur2018pitfalls,
  title={Pitfalls of graph neural network evaluation},
  author={Shchur, Oleksandr and others},
  journal={arXiv preprint arXiv:1811.05868},
  year={2018}
}

@article{wu2023sgformer,
  title={Sgformer: Simplifying and empowering transformers for large-graph representations},
  author={Wu, Qitian and others},
  journal={Advances in Neural Information Processing Systems},
  volume={36},
  pages={64753--64773},
  year={2023}
}

@book{WD,
  title={Optimal transport: old and new},
  author={Villani, C{\'e}dric and others},
  volume={338},
  year={2009},
}

@article{wgan-improved,
  title={Improved training of wasserstein gans},
  author={Gulrajani, Ishaan and others},
  journal={Advances in neural information processing systems},
  volume={30},
  year={2017}
}

@article{sinkhorn,
  title={Sinkhorn distances: Lightspeed computation of optimal transport},
  author={Cuturi, Marco},
  journal={Advances in neural information processing systems},
  volume={26},
  year={2013}
}

@article{singh2008global,
  title={Global alignment of multiple protein interaction networks with application to functional orthology detection},
  author={Singh, Rohit and Xu, Jinbo and Berger, Bonnie},
  journal={Proceedings of the National Academy of Sciences},
  volume={105},
  number={35},
  pages={12763--12768},
  year={2008},
  publisher={National Acad Sciences}
}

@inproceedings{zhang2015multiple,
  title={Multiple anonymized social networks alignment},
  author={Zhang, Jiawei and Philip, S Yu},
  booktitle={2015 IEEE International Conference on Data Mining},
  pages={599--608},
  year={2015},
  organization={IEEE}
}

@inproceedings{tang2008arnetminer,
  title={Arnetminer: extraction and mining of academic social networks},
  author={Tang, Jie and others},
  booktitle={Proceedings of the 14th ACM SIGKDD international conference on Knowledge discovery and data mining},
  pages={990--998},
  year={2008}
}

@inproceedings{zhang2021balancing,
  title={Balancing consistency and disparity in network alignment},
  author={Zhang, Si and others},
  booktitle={Proceedings of the 27th ACM SIGKDD conference on knowledge discovery \& data mining},
  pages={2212--2222},
  year={2021}
}

@inproceedings{li2019partially,
  title={Partially shared adversarial learning for semi-supervised multi-platform user identity linkage},
  author={Li, Chaozhuo and others   }, 
  booktitle={Proceedings of the 28th ACM international conference on information and knowledge management},
  pages={249--258},
  year={2019}
}

@inproceedings{sana,
  title={Robust Network Alignment with the Combination of Structure and Attribute Embeddings},
  author={Peng, Jingkai and others},
  booktitle={2023 IEEE International Conference on Data Mining (ICDM)},
  pages={498--507},
  year={2023},
  organization={IEEE}
}

@inproceedings{heimann2018regal,
  title={Regal: Representation learning-based graph alignment},
  author={Heimann, Mark and others },
  booktitle={Proceedings of the 27th ACM international conference on information and knowledge management},
  pages={117--126},
  year={2018}
}

@inproceedings{hermanns2021grasp,
  title={Grasp: Graph alignment through spectral signatures},
  author={Hermanns, Judith and others},
  booktitle={Web and Big Data: 5th International Joint Conference, APWeb-WAIM 2021, Guangzhou, China, August 23--25, 2021, Proceedings, Part I 5},
  pages={44--52},
  year={2021},
  organization={Springer}
}

@inproceedings{fusegwd,
  title={Optimal transport for structured data with application on graphs},
  author={Titouan, Vayer and others},
  booktitle={International Conference on Machine Learning},
  pages={6275--6284},
  year={2019},
  organization={PMLR}
}

@inproceedings{ppm,
  title={A fast proximal point method for computing exact wasserstein distance},
  author={Xie, Yujia and others},
  booktitle={Uncertainty in artificial intelligence},
  pages={433--453},
  year={2020},
  organization={PMLR}
}

@article{gin,
  title={How powerful are graph neural networks?},
  author={Xu, Keyulu and others},
  journal={arXiv preprint arXiv:1810.00826},
  year={2018}
}

@article{sinkhorn-knopp,
  title={Concerning nonnegative matrices and doubly stochastic matrices},
  author={Sinkhorn, Richard and others},
  journal={Pacific Journal of Mathematics},
  volume={21},
  number={2},
  pages={343--348},
  year={1967},
  publisher={Mathematical Sciences Publishers}
}

@article{wl-test,
  title={Weisfeiler-lehman graph kernels.},
  author={Shervashidze, Nino and others },
  journal={Journal of Machine Learning Research},
  volume={12},
  number={9},
  year={2011}
}

@inproceedings{haller2022comparative,
  title={A comparative study of graph matching algorithms in computer vision},
  author={Haller, Stefan and others},
  booktitle={European Conference on Computer Vision},
  pages={636--653},
  year={2022},
  organization={Springer}
}

@inproceedings{bernard2015solution,
  title={A solution for multi-alignment by transformation synchronisation},
  author={Bernard, Florian  and others},
  booktitle={Proceedings of the IEEE conference on computer vision and pattern recognition},
  pages={2161--2169},
  year={2015}
}

@article{tang2023fused,
  title={A fused gromov-wasserstein framework for unsupervised knowledge graph entity alignment},
  author={Tang, Jianheng and others},
  journal={arXiv preprint arXiv:2305.06574},
  year={2023}
}

@inproceedings{koutra2013big,
  title={Big-align: Fast bipartite graph alignment},
  author={Koutra, Danai and others},
  booktitle={2013 IEEE 13th international conference on data mining},
  pages={389--398},
  year={2013},
  organization={IEEE}
}

@inproceedings{zhang2020nettrans,
  title={Nettrans: Neural cross-network transformation},
  author={Zhang, Si and others},
  booktitle={Proceedings of the 26th ACM SIGKDD International Conference on Knowledge Discovery \& Data Mining},
  pages={986--996},
  year={2020}
}

@article{hong2020domain,
  title={Domain-adversarial network alignment},
  author={Hong, Huiting  and others},
  journal={IEEE Transactions on Knowledge and Data Engineering},
  volume={34},
  number={7},
  pages={3211--3224},
  year={2020},
  publisher={IEEE}
}

@inproceedings{zhou2021attent,
  title={Attent: Active attributed network alignment},
  author={Zhou, Qinghai and others},
  booktitle={Proceedings of the Web Conference 2021},
  pages={3896--3906},
  year={2021}
}

@article{zhang2022benchmark,
  title={A benchmark and comprehensive survey on knowledge graph entity alignment via representation learning},
  author={Zhang, Rui  and others},
  journal={The VLDB Journal},
  volume={31},
  number={5},
  pages={1143--1168},
  year={2022},
  publisher={Springer}
}

@inproceedings{wu2019simplifying,
  title={Simplifying graph convolutional networks},
  author={Wu, Felix and others},
  booktitle={International conference on machine learning},
  pages={6861--6871},
  year={2019},
  organization={PMLR}
}

@inproceedings{wang2018cross,
  title={Cross-lingual knowledge graph alignment via graph convolutional networks},
  author={Wang, Zhichun  and others},
  booktitle={Proceedings of the 2018 conference on empirical methods in natural language processing},
  pages={349--357},
  year={2018}
}

@inproceedings{sun2018bootstrapping,
  title={Bootstrapping entity alignment with knowledge graph embedding.},
  author={Sun, Zequn  and others},
  booktitle={IJCAI},
  volume={18},
  number={2018},
  year={2018}
}

@article{zeng2021comprehensive,
  title={A comprehensive survey of entity alignment for knowledge graphs},
  author={Zeng, Kaisheng and others},
  journal={AI Open},
  volume={2},
  pages={1--13},
  year={2021},
  publisher={Elsevier}
}

@article{kolavr2012graphalignment,
  title={GraphAlignment: Bayesian pairwise alignment of biological networks},
  author={Kol{\'a}{\v{r}}, Michal and Meier, J{\"o}rn  and others},
  journal={BMC systems biology},
  volume={6},
  pages={1--12},
  year={2012},
  publisher={Springer}
}

@inproceedings{sun2023towards,
  title={Towards higher-order topological consistency for unsupervised network alignment},
  author={Sun, Qingqiang and others},
  booktitle={2023 IEEE 39th International Conference on Data Engineering (ICDE)},
  pages={177--190},
  year={2023},
  organization={IEEE}
}

@article{liu2023wl,
  title={Wl-align: Weisfeiler-lehman relabeling for aligning users across networks via regularized representation learning},
  author={Liu, Li and others},
  journal={IEEE Transactions on Knowledge and Data Engineering},
  volume={36},
  number={1},
  pages={445--458},
  year={2023},
  publisher={IEEE}
}

@article{ning2023graph,
  title={Graph Alignment Neural Network Model With Graph to Sequence Learning},
  author={Ning, Nianwen and others},
  journal={IEEE Transactions on Knowledge and Data Engineering},
  year={2023},
  publisher={IEEE}
}

@article{huynh2021network,
  title={Network alignment with holistic embeddings},
  author={Huynh, Thanh Trung and others},
  journal={IEEE Transactions on Knowledge and Data Engineering},
  volume={35},
  number={2},
  pages={1881--1894},
  year={2021},
  publisher={IEEE}
}

@article{feizi2019spectral,
  title={Spectral alignment of graphs},
  author={Feizi, Soheil and others},
  journal={IEEE Transactions on Network Science and Engineering},
  volume={7},
  number={3},
  pages={1182--1197},
  year={2019},
  publisher={IEEE}
}

@inproceedings{karakasis2021joint,
  title={Joint graph embedding and alignment with spectral pivot},
  author={Karakasis, Paris A and others},
  booktitle={Proceedings of the 27th ACM SIGKDD Conference on Knowledge Discovery \& Data Mining},
  pages={851--859},
  year={2021}
}

@article{peng2010new,
  title={A new relaxation framework for quadratic assignment problems based on matrix splitting},
  author={Peng, Jiming and others},
  journal={Mathematical Programming Computation},
  volume={2},
  pages={59--77},
  year={2010},
  publisher={Springer}
}

@article{crouse2016implementing,
  title={On implementing 2D rectangular assignment algorithms},
  author={Crouse, David F and others},
  journal={IEEE Transactions on Aerospace and Electronic Systems},
  volume={52},
  number={4},
  pages={1679--1696},
  year={2016},
  publisher={IEEE}
}

@inproceedings{nassar2018low,
  title={Low rank spectral network alignment},
  author={Nassar, Huda and others},
  booktitle={Proceedings of the 2018 World Wide Web Conference},
  pages={619--628},
  year={2018}
}

@article{jiang2024unlocking,
  title={Unlocking the power of large language models for entity alignment},
  author={Jiang, Xuhui and others},
  journal={arXiv preprint arXiv:2402.15048},
  year={2024}
}

@article{chen2025enhancing,
  title={Enhancing robust semi-supervised graph alignment via adaptive optimal transport},
  author={Chen, Songyang and others},
  journal={World Wide Web},
  volume={28},
  number={2},
  pages={22},
  year={2025},
  publisher={Springer}
}

@inproceedings{AdaCoAgentEA,
  title={Towards Unsupervised Entity Alignment for Highly Heterogeneous Knowledge Graphs},
  author={Zhao, Runhao and others},
  booktitle={2025 IEEE 41st International Conference on Data Engineering (ICDE)},
  pages={3792--3806},
  year={2025},
  organization={IEEE Computer Society}
}

@article{LLM4EA,
  title={Entity alignment with noisy annotations from large language models},
  author={Chen, Shengyuan and others},
  journal={Advances in Neural Information Processing Systems},
  volume={37},
  pages={15097--15120},
  year={2024}
}

@inproceedings{HLOT,
  title={Hypergraph Learning for Unsupervised Graph Alignment via Optimal Transport},
  author={Yan, Yuguang and others},
  booktitle={Proceedings of the AAAI Conference on Artificial Intelligence},
  volume={39},
  number={20},
  pages={21913--21921},
  year={2025}
}

@article{piao2023computing,
  title={Computing graph edit distance via neural graph matching},
  author={Piao, Chengzhi  and others},
  journal={Proceedings of the VLDB Endowment},
  volume={16},
  number={8},
  pages={1817--1829},
  year={2023},
  publisher={VLDB Endowment}
}

@article{zhou2022metabolite,
  title={Metabolite annotation from knowns to unknowns through knowledge-guided multi-layer metabolic networking},
  author={Zhou, Zhiwei 
          and others
         },
  journal={Nature communications},
  volume={13},
  number={1},
  pages={6656},
  year={2022},
  publisher={Nature Publishing Group UK London}
}

@inproceedings{guo2025s3ga,
  title={S3GA: Towards Scalable Self-Supervised Learning for Large Scale Graph Alignment},
  author={Guo, Wenqi  and others},
  booktitle={Proceedings of the 31st ACM SIGKDD Conference on Knowledge Discovery and Data Mining V. 2},
  pages={755--766},
  year={2025}
}












\end{document}